\documentclass[runningheads]{llncs}
\usepackage[T1]{fontenc}
\usepackage{graphicx}

\usepackage{soul}
\usepackage{url}
\usepackage[hidelinks]{hyperref}
\usepackage[utf8]{inputenc}
\usepackage[small]{caption}
\usepackage{amsmath}
\usepackage{booktabs}
\usepackage{amssymb}
\usepackage{amsfonts}
\usepackage{thm-restate}
\usepackage{enumerate}
\usepackage{cleveref}
\usepackage[linesnumbered,ruled,vlined,noresetcount]{algorithm2e}
\usepackage[group-separator={,},binary-units]{siunitx}
\RequirePackage{mathtools}
\mathtoolsset{centercolon} %
\RequirePackage{xspace}
\RequirePackage{bm}

\newcommand{\ghost}[1]{\raisebox{0pt}[0pt][0pt]{\makebox[0pt][l]{#1}}}

\usepackage{todonotes}

\newcommand{\lang}[1]{\ensuremath{\mathbf{#1}}} %
\newcommand{\Clang}{\lang{C}} %
\newcommand{\Plang}{\lang{P}} %
\newcommand{\Vlang}{\lang{V}} %
\newcommand{\Nlang}{\lang{N}} %

\newcommand{\opfont}[1]{\text{\sf{#1}}} %
\renewcommand{\vec}[1]{\boldsymbol{#1}} %
\newcommand{\dom}{\opfont{dom}} %
\newcommand{\im}{\opfont{im}} %

\newcommand{\arity}{\opfont{ar}} %
\newcommand{\Dnter}{\mathcal{D}} %
\newcommand{\Inter}{\mathcal{I}} %
\newcommand{\arule}{\rho} %
\newcommand{\rbody}{\opfont{body}} %
\newcommand{\rhead}{\opfont{head}} %
\newcommand{\aprogram}{\Sigma} %

\newcommand{\rbodyA}{\rbody_1}
\newcommand{\rbodyB}{\rbody_2}
\newcommand{\rbodyi}{\rbody_i}
\newcommand{\rheadA}{\rhead_1}
\newcommand{\rheadB}{\rhead_2}
\newcommand{\rheadi}{\rhead_i}

\newcommand{\relres}{\mathrel{\prec^{\square}}}
\newcommand{\relpos}{\mathrel{\prec^{+}}}

\newcommand{\complclass}[1]{\text{{\sc #1}}\xspace} %

\newcommand{\NP}{\complclass{NP}}

\newcommand{\SigmaTwoP}{\ensuremath{\Sigma_{\complclass{2}}^{\complclass{P}}}}
\newcommand{\TwoExpTime}{\complclass{2ExpTime}}

\usepackage{comment}
\includecomment{paper}
\includecomment{tr}
\excludecomment{paper}

\title{Efficient Dependency Analysis\\ for Rule-Based Ontologies}

\begin{paper}
\author{Larry González\orcidID{0000-0001-9412-9363} \and
    Alex Ivliev\orcidID{0000-0002-1604-6308} \and
    Markus Krötzsch\orcidID{0000-0002-9172-2601} \and
    Stephan Mennicke\orcidID{0000-0002-3293-2940}}
\end{paper}
\institute{Knowledge-Based Systems Group, TU Dresden, Germany\\ \email{firstname.lastname@tu-dresden.de}}

\begin{tr}
  \subtitle{Extended Version}
\author{Larry González \and
    Alex Ivliev \and
    Markus Krötzsch \and
    Stephan Mennicke}
\end{tr}

\newcommand\numOfRuleSets{201\xspace}
\newcommand\numOfOntologies{787\xspace}

\newcommand\imp[1][N]{\textsf{#1}\xspace}

\begin{document}
\maketitle

\begin{abstract}
Several types of \emph{dependencies} have 
been proposed for the static analysis of existential rule ontologies, promising insights about 
computational properties and possible practical uses of a given set of rules, e.g., in ontology-based query answering.
Unfortunately, these dependencies are rarely implemented,
so their potential is hardly realised in practice.
We focus on two kinds of rule dependencies -- \emph{positive reliances} and \emph{restraints} --
and design and implement optimised algorithms for their efficient computation.
Experiments on real-world ontologies of up to more than 100,000 rules show the scalability
of our approach, which lets us realise several previously proposed applications
as practical case studies. In particular, we can analyse
to what extent rule-based bottom-up approaches of reasoning can be guaranteed to yield
redundancy-free ``lean'' knowledge graphs (so-called \emph{cores}) on practical ontologies.
\end{abstract}
\section{Introduction}

\emph{Existential rules} are a versatile knowledge representation language 
with relevance
in ontological reasoning \cite{BLMS11:decline,CaliGL12:datalogplusontologies,CGP12:stickyAIJ,CG+13:acyclicity},
databases \cite{FaginKMP05,DNR08:corechase,GO:chase18}, and declarative computing in general
\cite{BSG:Vadalog18,Carral+19:ChasingSets,BCKRT:homClosedDecExRules2021}.
In various semantic web applications, existential rule engines have been used to process knowledge graphs and ontologies,
often realising performance advantages on large data sets
\cite{BLMRS15:Graal,VLog4j2019,N+15:RDFoxToolPaper,BSG:Vadalog18}.

Existential rules extend the basic rule language Datalog with the facility for \emph{value invention}, expressed by existentially quantified variables in conclusions. This ability to refer to ``unknown'' values is an important similarity to description logics (DLs) and
the DL-based ontology standard OWL, and many such ontologies can equivalently be expressed in existential rules.
This can be a practical approach for ontology-based query answering \cite{CG+13:acyclicity,CDK2018:combinedHornALCHOIQ}.
For reasoning, many rule engines rely on \emph{materialisation}, where the input data is expanded iteratively until all
rules are satisfied (this type of computation is called \emph{chase}).
With existentials, this can require adding new ``anonymous'' individuals -- called \emph{nulls} --,
and the process may not terminate. Several \emph{acyclicity conditions} define cases where termination is ensured,
and were shown to apply to many practical ontologies~\cite{CG+13:acyclicity}.

Nulls correspond to \emph{blank nodes} in RDF, and -- like bnodes in RDF \cite{Mallea2011:blanks} -- are not always desirable.
Avoiding nulls entirely is not an option in chase-based reasoning, but one can still avoid some ``semantically redundant'' nulls.
For example, given a fact $\textsf{person}(\textsf{alice})$ and a rule $\textsf{person}(x)\to\exists y.\, \textsf{parent}(x,y)$,
the chase would derive $\textsf{parent}(\textsf{alice},n)$ for a fresh null $n$. However, if we already know that
$\textsf{parent}(\textsf{alice},\textsf{bob})$, then this inference is redundant and can be omitted.
In general, structures that are free of such redundancies are mathematically known as \emph{cores}. An RDF-graph that
is a core is called a \emph{lean} graph \cite{Hogan2017:leanRDF}.
Unfortunately, the computation of cores is expensive, and can in general not be afforded during the chase.
Sometimes, however, when rules satisfy a condition known as \emph{core stratification}, practical chase
algorithms can also produce a core directly \cite{KR20:cores}.

Interestingly, both of the previously mentioned types of conditions -- acyclicity and core stratification -- are detected
by analysing \emph{dependencies}\footnote{We use the term only informally,
since \emph{(tuple-generating) dependencies} are also a common name for rules in databases.}
that indicate possible semantic interactions between rules.
Early works focussed on cases where a rule $\rho_2$ \emph{positively relies} on a rule $\rho_1$ in the sense that
an application of rule $\rho_1$ might trigger an application of rule $\rho_2$.
They are used to detect several forms of acyclity \cite{BLMS11:decline,DNR08:corechase,MSL:chaseTerminationStratification09}.
When adding negation, a rule might also inhibit another, and such \emph{negative reliances}
are used to define semantically well-behaved fragments of nonmonotonic existential rules \cite{KR20:cores,MKH13:reliances}.
A third kind of dependency are \emph{restraints}, which indicate that the application of
one rule might render another one redundant:
restraints were used to define \emph{core stratified rule sets} \cite{KR20:cores},
and recently also to define a semantics for queries with negation \cite{EKM2022:BNCQs}.
Definitions for these various dependencies have many commonalities.

Surprisingly, given this breadth of applications, rule dependencies are hardly supported in practice.
To our knowledge, positive reliances are only computed by the Graal toolkit \cite{BLMRS15:Graal}, whereas negative reliances and
restraints have no implementation at all.
A possible reason is that such dependency checks are highly intractable, typically $\SigmaTwoP$-complete,
and therefore not easy to implement efficiently. This is critical since their proposed uses are often
related to the choice of a rule-processing strategy, so that their computation adds to overall
reasoning time. Moreover, as opposed to many other static analyses, dependency computation is not mainly
an application of algorithms that are already used in rule reasoning.
Today's use of dependencies in optimisation and analysis therefore falls short of expectations.

To address this problem, we design optimised algorithms for the computation of 
positive reliances and restraints. We propose \emph{global} optimisations, which reduce the number of relevant checks,
and \emph{local} optimisations, which reduce the work needed to execute a specific check.
The latter include an improved search strategy that often avoids the full exploration of
exponentially many subsets of rule atoms, which may be necessary in the worst case.
The underlying ideas can also be adapted to negative reliances and any of the modified definitions of
positive reliances found in the literature.

We implement our methods and conduct extensive experiments with over 200 real-world ontologies of varying sizes.
Considering the effectiveness of our optimisations, we find that local and global techniques both make important contributions
to overall performance. The performance of our prototype enables various practical uses:
\begin{itemize}
\item We conduct the first analysis of the practical prevalence of \emph{core stratification} \cite{KR20:cores} using our implementation of restraints.
  We find this desirable property in a significant share of ontologies from a curated repository and provide preliminary insights on why some rule sets are not core stratified.
\item Comparing the computation of all positive reliances to Graal, we see speed-ups of more than two orders of magnitude.
 Our stronger definition yields an \emph{acyclic graph of rule dependencies} \cite{BLMS11:decline} in more cases.
\item Using the graph of positive reliances, we show how to speed up the popular but expensive rule analysis algorithm \emph{MFA} \cite{CG+13:acyclicity}.
   Compared to the MFA implementation of VLog \cite{VLog4j2019}, we observe speed-ups of up to four orders of magnitude.
   Thereby, we also show that an incorporation of our algorithms into rule reasoners is feasible.
\end{itemize}
\section{Preliminaries}\label{sec_prelims}

We build expressions from countably infinite, mutually disjoint
sets
$\Vlang$ of \emph{variables},
$\Clang$ of \emph{constants},
$\Nlang$ of \emph{labelled nulls}, and
$\Plang$ of \emph{predicate names}.
Each predicate name $p\in\Plang$ has an \emph{arity} $\arity(p)\geq 0$.
\emph{Terms} are elements of $\Vlang\cup\Nlang\cup\Clang$.
We use $\vec{t}$ to denote a list $t_1,\ldots,t_{|\vec{t}|}$ of terms, and similar for special types of terms.
An \emph{atom} is an expression $p(\vec{t})$ with $p\in\Plang$, $\vec{t}$ a list of terms, and $\arity(p)=|\vec{t}|$.
\emph{Ground} terms or atoms contain neither variables nor nulls.
An \emph{interpretation} $\Inter$ is a set of atoms without variables.
A \emph{database} $\Dnter$ is a finite set of ground atoms (i.e., a finite interpretation without nulls).

\subsubsection*{Syntax}
An \emph{existential rule} (or just \emph{rule}) $\arule$
is a formula
\begin{align}
\arule = \forall \vec{x}, \vec{y}.\, \varphi[\vec{x}, \vec{y}] \to
\exists \vec{z}.\, \psi[\vec{y}, \vec{z}],\label{eq_rule}
\end{align}
where $\varphi$ and $\psi$ are
conjunctions of atoms using only terms from $\Clang$ or from the
mutually disjoint lists of variables $\vec{x}, \vec{y}, \vec{z}\subseteq\Vlang$. 
We call $\varphi$ the \emph{body} (denoted $\rbody(\arule)$) and $\psi$ the \emph{head} (denoted $\rhead(\arule)$).
We may treat conjunctions of atoms as sets, and we omit universal quantifiers in rules.
We require that all variables in $\vec{y}$ do really occur in $\varphi$ (\emph{safety}). %
A rule is \emph{Datalog} if it has no existential quantifiers.

\subsubsection*{Semantics}
Given a set of atoms $\mathcal{A}$ and an interpretation $\Inter$,
a \emph{homomorphism} $h\colon\mathcal{A}\to\Inter$ is a function that maps the terms occurring in $\mathcal{A}$
to (the variable-free) terms occurring in $\Inter$,
such that:
(i) for all $c\in\Clang$, $h(c)=c$;
(ii) for all $p\in\Plang$, $p(\vec{t})\in\mathcal{A}$ implies $p(h(\vec{t}))\in\Inter$, where $h(\vec{t})$ is the list of $h$-images of the terms $\vec{t}$.
If, in addition, $p(h(\vec{t}))\in\Inter$ implies $p(\vec{t})\in\mathcal{A}$ (i.e., (ii) can be strengthened to an ``if, and only if''), then $h$ is \emph{strong}.
We apply homomorphisms to a formula by applying them individually to all of its terms.

A \emph{match} of a rule $\arule$ in an interpretation $\Inter$ is
a homomorphism $\rbody(\rho)\to\Inter$.
A match $h$ of $\arule$ in $\Inter$ is \emph{satisfied} if there is a homomorphism $h'\colon\rhead(\arule)\to\Inter$ that agrees
with $h$ on all variables that occur in body and head (i.e., variables $\vec{y}$ in \eqref{eq_rule}).
Rule $\arule$ is \emph{satisfied} by $\Inter$, written $\Inter\models\arule$, if every match
of $\arule$ in $\Inter$ is satisfied.
A set of rules $\aprogram$ is satisfied by $\Inter$, written $\Inter\models\aprogram$, if
 $\Inter\models\arule$ for all $\arule\in\aprogram$.
We may treat databases $\Dnter$ as sets of rules with empty bodies (also called \emph{facts}), and write, e.g.,
$\Inter\models \Dnter,\aprogram$ to express that $\Inter\models\aprogram$ and $\Dnter\subseteq\Inter$.
In this case, $\Inter$ is a \emph{model} of $\aprogram$ and $\Dnter$.

\subsubsection*{Applying rules}
A rule $\arule$ of form \eqref{eq_rule} is \emph{applicable} to an interpretation $\Inter$
if there is an unsatisfied match $h$ in $\Inter$ (i.e., $h$ cannot be extended to a homomorphism $\psi\to\Inter$).
The result of applying $\arule$ for $h$ is the interpretation $\Inter \cup \psi[h'(\vec{y}),h'(\vec{z})]$,
where $h'$ is a mapping such that $h'(y)=h(y)$ for all $y\in\vec{y}$, and for all $z\in\vec{z}$, $h'(z)\in\Nlang$ is a distinct  null not occurring in $\Inter$.
The \emph{(standard) chase} is a reasoning algorithm obtained by applying rules to a given initial database, 
such that all applicable rules are eventually applied (fairness).

\subsubsection*{Core models}
A model $\Inter$ is a \emph{core} if every homomorphism $h\colon \Inter \to \Inter$
is strong and injective. For finite models, this is equivalent to the requirement that every such homomorphism is an isomorphism,
and this will be the only case we are interested in for this work.
Intuitively, the condition states that the model does not contain a strictly smaller substructure that is
semantically equivalent for conjunctive query answering.

\subsubsection*{Unification}
For atom sets $\mathcal A$ and $\mathcal B$, partial function $m \colon \mathcal A\to\mathcal B$ is an \emph{atom mapping}, where $\dom(m)\subseteq\mathcal A$ is the set of all atoms for which $m$ is defined.
A \emph{substitution} is a function $\theta \colon \Clang\cup\Vlang\cup\Nlang \to \Clang\cup\Vlang\cup\Nlang$, such that $\theta(c)=c$ for all $c\in\Clang\cup\Nlang$.
Denote the application of $\theta$ to term $t$ by $t\theta$, naturally extending to atoms and atom sets by term-wise application.
The concatenation of substitutions $\sigma$ and $\theta$ is $\sigma\theta$ where $t\sigma\theta = (t\sigma)\theta$.
A substitution is a \emph{unifier} for atom mapping $m$ if for all $\alpha\in\dom(m)$, $\alpha\theta = (m(\alpha))\theta$.
A unifier $\mu$ for $m$ is a \emph{most general unifier} (mgu) for $m$ if for all unifiers $\nu$ of $m$, there is a substitution $\sigma$, such that $\mu\sigma=\nu$.

\section{Dependencies and their naive computation}\label{sec_deps}

We first introduce the two kinds of rule dependencies that we consider:
\emph{positive reliances} and \emph{restraints}.
Our definitions largely agree with the literature, but there are some
small differences that we comment on.

\begin{definition}\label{def_posrel}
A rule $\arule_2$ \emph{positively relies on} a rule $\arule_1$, written $\arule_1\relpos\arule_2$,
if there are interpretations $\Inter_a\subseteq\Inter_b$ and
a function $h_2$ such that
\begin{enumerate}[(a)]
\item \label{it_pos_amatch} $\Inter_b$ is obtained from $\Inter_a$ by applying $\arule_1$ for the match $h_1$ extended to $h_1'$,
\item \label{it_pos_unsat} $h_2$ is an unsatisfied match for $\arule_2$ on $\Inter_b$, and
\item \label{it_pos_nomatch} $h_2$ is not a match for $\arule_2$ on $\Inter_a$.
\end{enumerate}
\end{definition}

Definition~\ref{def_posrel} describes a situation where an application of $\arule_1$
immediately enables a new application of $\arule_2$. Condition \eqref{it_pos_unsat}
takes into account that only unsatisfied matches can lead to rule applications in the standard chase.
The same condition is used by Krötzsch \cite{KR20:cores}, whereas
Baget et al. \cite{BLMS11:decline,BLMRS15:Graal} -- using what they call \emph{piece-unifier} -- only require $h_2$ to be a match.
In general, weaker definitions are not incorrect, but may lead to unnecessary dependencies.

\begin{example}\label{ex_relpos_bleeding}
  Consider the following ontology.
  We provide three axioms in DL syntax (left-hand side) and their translation into existential rules (right-hand side).
  \begin{align*}
    A &\sqsubseteq \exists R.B   & a(x) &\rightarrow \exists v.\, r(x, v) \wedge b(v) \tag{$\arule_1$}\\
    R^- \circ R &\sqsubseteq T   & r(y, z_1) \wedge r(y, z_2) &\rightarrow t(z_1, z_2) \tag{$\arule_2$}\\
    \exists R^-.A &\sqsubseteq B & a(t) \wedge r(t, u) &\rightarrow b(u) \tag{$\arule_3$}
  \end{align*}

  \noindent For this rule set,
  we find $\arule_1 \relpos \arule_2$ by using
  $\Inter_a = \{ a(c) \}$, 
  $\Inter_b = \{ a(c), r(c,n) \}$, 
  and $h_2 = \{ y\mapsto c, z_1\mapsto n, z_2\mapsto n \}$.
  Note that $\arule_3$ does not positively rely on $\arule_1$ although the application 
  of $\arule_1$ may lead to a new match for $\arule_3$.
  However, this match is always satisfied, so condition (\ref{it_pos_unsat}) of 
  Definition~\ref{def_posrel} is not fulfilled.
\end{example}

The definition of restraints considers situations where the nulls introduced by applying rule
$\arule_2$ are at least in part rendered obsolete by a later application of $\arule_1$.
This obsolescence is witnessed by an \emph{alternative match} that specifies a different way of
satisfying the rule match of $\arule_2$.

\begin{definition}\label{def_altmatch}
Let $\Inter_a\subseteq\Inter_b$ be interpretations
such that $\Inter_a$ was obtained
by applying the rule $\arule$ for match $h$ which is extended to $h'$.
A homomorphism $h^A\colon h'(\rhead(\arule))\to \Inter_b$
is an \emph{alternative match} of $h'$ and $\rho$ on $\Inter_{b}$ if
\begin{enumerate}[(1)]
\item \label{it_altmatch_samebody} $h^A(t)=t$ for all terms $t$ in $h(\rbody(\arule))$, and
\item \label{it_altmatch_newhead} there is a null $n$ in $h'(\rhead(\arule))$ that does not occur in $h^A(h'(\rhead(\arule)))$.
\end{enumerate}
\end{definition}

Now $\arule_1$ restrains $\arule_2$ if it creates an alternative match for it:

\begin{definition}\label{def_restrained}
A rule $\arule_1$ \emph{restrains} a rule $\arule_2$, written $\arule_1\relres\arule_2$,
if there are interpretations $\Inter_a\subseteq\Inter_b$ such that
\begin{enumerate}[(a)]
\item \label{it_rest_bmatch} $\Inter_b$ is obtained by applying $\arule_1$ for match $h_1$ extended to $h_1'$,
\item \label{it_rest_amatch} $\Inter_a$ is obtained by applying $\arule_2$ for match $h_2$ extended to $h_2'$,
\item \label{it_rest_alt} there is an alternative match $h^A$ of $h_2$ and $\rho_{2}$ on $\Inter_{b}$, and
\item \label{it_rest_noalt} $h^A$ is no alternative match of $h_2$ and $\rho_{2}$ on $\Inter_b\setminus h_1'(\rhead(\arule_1))$.
\end{enumerate}
\end{definition}

Our definition slightly deviates from the literature \cite{KR20:cores},
where \eqref{it_rest_noalt} made a stronger requirement:
\begin{enumerate}[{(d')}]
\item $h_2$ has no alternative match $h_2'(\rhead(\arule_2))\to\Inter_b\setminus h_1'(\rhead(\arule_1))$.
\end{enumerate}
As we will see, our modification allows for a much more efficient implementation,
but it also leads to more restraints. Since restraints overestimate potential 
interactions during the chase anyway, all formal results of prior works are preserved.

\begin{example}\label{ex_different_restraints}
For the rules 
$\arule_1= r(y,y)\to \exists w.\, r(y,w)\wedge b(w)$ and
$\arule_2= a(x)\to \exists v.\, r(x,v)$,
we find $\arule_1\relres\arule_2$ by Definition~\ref{def_restrained},
where we set $\Inter_a=\{a(c),r(c,n_1)\}$, 
$\Inter_b=\Inter_a\cup\{r(c,c),r(c,n_2),b(n_2)\}$,
and $h^A=\{c\mapsto c, n_1\mapsto n_2\}$.
However, these $\Inter_a$ and $\Inter_b$ do not satisfy the stricter condition (d'), 
since $h^B=\{c\mapsto c, n_1\mapsto c\}$ is an alternative match, too.
Indeed, when $\arule_2$ is applicable in such a 
way as to produce an alternative match 
w.r.t. an application of $\arule_1$, 
another one must have already existed.
\end{example}

Example~\ref{ex_different_restraints} is representative of situations
where \eqref{it_rest_noalt} leads to different restraints than (d'):
the body of the restraining rule $\arule_1$ must contain a pattern that enforces an additional
alternative match (here: $r(y,y)$), while not being satisfiable by the conclusion of $\arule_2$ (here: $r(y,n_1)$).
To satisfy the remaining conditions, $\rhead(\arule_1)$ must further produce a (distinct) alternative match.
Such situations are very rare in practice, so that the benefits of \eqref{it_rest_noalt} outweigh the loss of generality.

Checking for positive reliances and restraints is $\SigmaTwoP$-complete.
Indeed, we can assume $\Inter_a$ and $\Inter_b$ to contain at most as many elements
as there are distinct terms in the rule, so that they can be polynomially guessed.
The remaining conditions can be checked by an $\NP$-oracle.
Hardness follows from the $\SigmaTwoP$-hardness of deciding if a rule has an
unsatisfied match \cite{GO:chase18}.

The existence of alternative matches in a chase sequence 
indicates that the resulting model may contain redundant nulls.
Ordering the application of rules during the chase in a way that obeys
the restraint relationship ($\relres$)
ensures that the chase sequence does not contain any alternative matches
and therefore results in a core model \cite{KR20:cores}.

\begin{example}
  Consider again the rule set from Example~\ref{ex_relpos_bleeding}.
  For the interpretation $\Inter_0 = \{a(c), r(c, d)\}$
  all three rules are applicable.
  Disregarding $\arule_3 \relres \arule_1$ 
  and applying $\arule_1$ first results in $\Inter_1 = \Inter_0 \cup \{r(c, n), b(n)\}$,
  which leads to the alternative match $h^A = \{c \mapsto c, n \mapsto d\}$ 
  after applying $\arule_3$.
  If we, on the other hand,
  start with $\arule_3$,
  we obtain $\Inter_1' = \Inter_0 \cup \{b(d)\}$.
  Rule $\arule_1$ is now satisfied and 
  the computation finishes with a core model after applying $\arule_2$.
\end{example}

The ontology from Example~\ref{ex_relpos_bleeding} 
is an example of a \emph{core stratified} rule set.
A set of rules is \emph{core stratified} 
if the graph of all $\relpos \cup \relres$
edges does not have a cycle that includes a $\relres$ edge.
This property allows us to formulate a rule application strategy that respects the restraint relationship
as follows:
Given $\arule_1 \relres \arule_2$, apply the restrained rule $\rho_2$ only 
if neither $\rho_1$ nor any of the rules 
$\rho_1$ directly or indirectly positively relies on is applicable.

\section{Computing positive reliances}\label{sec_relpos}

The observation that positive reliances can be decided in $\SigmaTwoP$ 
is based on an algorithm that considers all possible sets $\Inter_a$ and $\Inter_b$
up to a certain size. This is not practical, in particular for uses
where dependencies need to be computed as part of the (performance-critical) reasoning,
and we therefore develop a more goal-oriented approach.

In the following, we consider two rules $\arule_1$ and $\arule_2$
of form $\arule_i = \rbodyi \to \exists \vec{z}_i.\, \rheadi$,
with variables renamed so that no variable occurs in both rules.
Let $\Vlang_\forall$ and $\Vlang_\exists$, respectively, denote the sets of universally
and existentially quantified variables in $\arule_1$ and $\arule_2$.
A first insight is that the sets $\Inter_a$ and $\Inter_b$ of
Definition~\ref{def_posrel} can be assumed to contain only atoms
that correspond to atoms in $\arule_1$ and $\arule_2$, with distinct
universal or existential variables replaced by distinct constants or nulls,
respectively. For this replacement, we fix a substitution $\omega$
that maps each variable in $\Vlang_\exists$ to a distinct
null, and each variable in $\Vlang_\forall$ to a distinct
constant that does not occur in $\arule_1$ or $\arule_2$.

A second insight is that, by \eqref{it_pos_nomatch}, $\arule_1$ must produce some atoms that are relevant
for a match of $\arule_2$, so that our algorithm can specifically search for 
a \emph{mapped} subset $\rbody^m_2\subseteq\rbodyB$ and a substitution $\eta$
such that $\rbody^m_2\eta\subseteq\rheadA\eta$.
Note that $\eta$ represents both matches $h_1$ and $h_2$ from Definition~\ref{def_posrel},
which is possible since variables in $\arule_1$ and $\arule_2$ are disjoint.
The corresponding set $\Inter_a$ then is
$(\rbodyA\cup(\rbodyB\setminus\rbody^m_2))\eta\omega$.
Unfortunately, it does not suffice to consider singleton sets for $\rbody^m_2$, as shown by Example~\ref{ex_bleeding_contd}:
\begin{example}\label{ex_bleeding_contd}
  Consider the rules from Example~\ref{ex_relpos_bleeding}.
  Trying to map either one of the atoms of $\rbody(\arule_2)$ to $\rhead(\arule_1)$ yields an $\Inter_a=\{ a(c), r(c,c') \}$,
  to which $\arule_1$ is not applicable.
  The correct $\Inter_a=\{a(c)\}$ as given in Example~\ref{ex_relpos_bleeding} is found by unifying both atoms of $\rbody(\arule_2)$ with (an instance of) $\rhead(\arule_1)$.
\end{example}
Therefore, we have to analyse all subsets $\rbody^m_2\subseteq\rbodyB$ for possible matches
with $\rheadA$.
We start the search from singleton sets, which are successively extended by adding atoms.
A final important insight is that this search can often be aborted early, since
a candidate pair for $\Inter_a$ and $\Inter_b$ may fail Definition~\ref{def_posrel} for
various reasons, and considering a larger $\rbody^m_2$ is not always promising.
For example, if $\eta$ is a satisfied match for $\arule_2$ over $\Inter_b$ \eqref{it_pos_unsat},
then adding more atoms to $\rbody^m_2$ will never succeed.

\DontPrintSemicolon
\SetKwFunction{checkpos}{check$^+$}
\SetKwFunction{extendpos}{extend$^{+\!}$}
\SetKwFunction{unify}{unify}
\SetKwFunction{maxidx}{maxidx}
\begin{algorithm}[tb] \caption{\FuncSty{extend$^+$}}\label{alg_extend_positive}

  \KwIn{rules $\arule_1,\arule_2$, atom mapping $m$}
  \KwOut{\textit{true} iff the atom mapping can be extended successfully}

  \For{$i \in \{\maxidx{m}+1, \dots, |\rbodyB|\}$} {\label{alg:extend:for_i}
    \For{$j  \in \{1, \dots, |\rheadA|\}$}{
      $m' \leftarrow m \cup \{\rbodyB[i] \mapsto \rheadA[j]\omega_\exists\}$\; \label{alg:extend:mprime}
      
      \If{$\eta \leftarrow{}$\unify{$m'$}} { \label{alg:extend:unify}
        \lIf{\checkpos{$\arule_1$,$\arule_2$,$m'$,$\eta$}} {\label{alg:extend:check}%
          \ghost{\KwRet{true}}
        }
      }
    }
  }

  \KwRet{false}\;
\end{algorithm}

\begin{algorithm}[tb] \caption{\FuncSty{check$^+$}}\label{alg_check_positive}

  \KwIn{rules $\arule_1,\arule_2$, atom mapping $m$ with mgu $\eta$}
  \KwOut{\textit{true} if a positive reliance is found for $m$}

  $\rbody^m_2 \leftarrow \dom(m)$\; \label{alg:check_positive:defphi21}
  $\rbody^\ell_2 \leftarrow \ghost{$\{\rbodyB[j]\in(\rbodyB{\setminus}\,\rbody^m_2)\mid j\,{<}\,\maxidx{m}\}$}$\; \label{alg:check_positive:defphi22left}
  $\rbody^r_2 \leftarrow \ghost{$\{\rbodyB[j]\in(\rbodyB{\setminus}\,\rbody^m_2)\mid j\,{>}\,\maxidx{m}\}$}$\; \label{alg:check_positive:defphi22right}
  \lIf{$\rbodyA\eta$ contains a null}{\label{alg:check_positive:noNulls1}%
   \ghost{\KwRet{false}}\label{alg:check_positive:noNulls1_return}
  }
  \lIf{$\rbody^\ell_2\eta$ contains a null}{%
    \KwRet{false}\label{alg:check_positive:noNulls22_left}
  }
  \lIf{$\rbody^r_2\eta$ contains a null}{\label{alg:check_positive:noNulls22}%
    \ghost{\KwRet{\extendpos{$\arule_1$,$\arule_2$,$m$}}} \label{alg:check_positive:noNulls22_return}
  }

  $\mathcal{I}_a \leftarrow  (\rbodyA \cup \rbody^\ell_2\cup \rbody^r_2)\eta\omega$\; \label{alg:check_positive:Ia}

  \lIf{$\mathcal{I}_a \models  \exists\vec{z}_1.\, \rheadA\eta\omega_\forall$}{\label{alg:check_positive:modelsPsi1}%
    \ghost{\KwRet{\extendpos{$\arule_1$,$\arule_2$,$m$}}}
  }

  \lIf{$\rbodyB\eta\omega \subseteq \mathcal{I}_a$}{\label{alg:check_positive:modelsPhi2}%
    \ghost{\KwRet{\extendpos{$\arule_1$,$\arule_2$,$m$}}}
  }

  $\mathcal{I}_b \leftarrow \mathcal{I}_a \cup \rheadA\eta\omega$\; \label{alg:check_positive:Ib}

  \lIf{$\mathcal{I}_b \models \exists\vec{z}_2.\, \rheadB\eta\omega_\forall$ }{\label{alg:check_positive:modelsPsi2}%
    \KwRet{false} \label{alg:check_positive:modelsPsi2_return}
  }

  \KwRet{true}\; \label{alg:check_positive:return_true}
\end{algorithm}
These ideas are implemented in Algorithms~\ref{alg_extend_positive} (\FuncSty{extend$^+$}) and \ref{alg_check_positive}
(\FuncSty{check$^+$}), explained next.
For a substitution $\theta$, we write $\theta_\forall$ ($\theta_\exists$,
resp.), to denote
the substitution assigning existential variables (universal variables, resp.) to themselves,
and otherwise agrees with $\theta$.

Function \FuncSty{extend$^+$} iterates over extensions of a given candidate set.
To specify how atoms of $\rbodyB$ are mapped to $\rheadA$, we maintain
an atom mapping $m\colon \rbodyB\to\rheadA$
whose domain $\dom(m)$ corresponds to the chosen $\rbody^m_2\subseteq\rbodyB$.
To check for the positive reliance, we initially call \extendpos{$\arule_1$,$\arule_2$,$\emptyset$}.
Note that $\arule_1$ and $\arule_2$ can be based on the same rule (a rule can positively rely on
itself); we still use two variants that ensure disjoint variable names.

We treat rule bodies and heads as lists of atoms, and write $\varphi[i]$ for the
$i$th atom in $\varphi$. The expression $\maxidx{m}$ returns the largest index of an
atom in $\dom(m)$, or $0$ if $\dom(m)=\emptyset$. By extending $m$ only with atoms of larger
index (L\ref{alg:extend:for_i}), we ensure that each $\dom(m)$ is only considered
once. We then construct each possible extension of $m$ (L\ref{alg:extend:mprime}),
where we replace existential variables by fresh nulls in $\rheadA$.
In Line~\ref{alg:extend:unify}, \unify{$m'$} is the most general unifier $\eta$ of $m'$
or undefined if $m'$ cannot be unified.
With variables, constants, and nulls as the only terms, unification is an easy polynomial algorithm.

Processing continues with \FuncSty{check$^+$}, called in Line~\ref{alg:extend:check} of \FuncSty{extend$^+$}.
We first partition $\rbodyB$ into the matched atoms $\rbody^m_2$, and the remaining atoms to the left $\rbody^\ell_2$ and right $\rbody^r_2$
of the maximal index of $m$. Only $\rbody^r_2$ can still be considered for extending $m$.
Six if-blocks check all conditions of Definition~\ref{def_posrel}, and \textit{true} is returned if all checks succeed.
When a check fails, the search is either stopped (L\ref{alg:check_positive:noNulls1}, L\ref{alg:check_positive:noNulls22_left}, and L\ref{alg:check_positive:modelsPsi2})
or recursively continued with an extended mapping (L\ref{alg:check_positive:noNulls22},
L\ref{alg:check_positive:modelsPsi1}, and L\ref{alg:check_positive:modelsPhi2}).
The three checks in L\ref{alg:check_positive:noNulls1}--L\ref{alg:check_positive:noNulls22}
cover cases where $\Inter_a$ (L\ref{alg:check_positive:Ia}) would need to contain nulls that are
freshly introduced by $\arule_1$ only later.
L\ref{alg:check_positive:noNulls1} applies, e.g., when checking
$\arule_2\relpos\arule_1$ for $\arule_1,\arule_2$ as in Example~\ref{ex_different_restraints},
where we would get $a(n)\in\Inter_a$ (note the swap of rule names compared to our present algorithm).
Further extensions of $m$ are useless for L\ref{alg:check_positive:noNulls1}, since they could only lead to more specific unifiers,
and also for L\ref{alg:check_positive:noNulls22_left}, where nulls occur in ``earlier'' atoms that are not considered in extensions of $m$.
For case L\ref{alg:check_positive:noNulls22}, however, moving further atoms from $\rbody^r_2$ to $\rbody^m_2$ might be promising,
so we call \FuncSty{extend$^+$} there.

In L\ref{alg:check_positive:modelsPsi1}, we check if the constructed match of $\arule_1$ on $\Inter_a$ is already satisfied.
This might again be fixed by extending the mapping, since doing so makes $\rbody^r_2$ and hence $\Inter_a$ smaller.
If we reach L\ref{alg:check_positive:modelsPhi2}, we have established condition \eqref{it_pos_amatch} of Definition~\ref{def_posrel}.
L\ref{alg:check_positive:modelsPhi2} then ensures condition \eqref{it_pos_nomatch}, which might
again be repaired by extending the atom mapping so as to make $\Inter_a$ smaller.
Finally, L\ref{alg:check_positive:modelsPsi2} checks condition \eqref{it_pos_unsat}. If this fails,
we can abort the search: unifying more atoms of $\rbodyB$ with $\rheadA$ will only lead
to a more specific $\Inter_b$ and $\eta$, for which the check would still fail.

\begin{restatable}{theorem}{thmRelPosCorrect}\label{theo_relpos_correct}
For rules $\arule_1$ and $\arule_2$ that (w.l.o.g.) do not share variables,
$\arule_1\relpos\arule_2$ iff \extendpos{$\arule_1$,$\arule_2$,$\emptyset$}${}=\textit{true}$.  
\end{restatable}

\section{Computing restraints}\label{sec_relres}

\SetKwFunction{checkres}{check$^\square$}
\SetKwFunction{extendres}{extend$^\square$}
We now turn our attention to the efficient computation of restraints.
In spite of the rather different definitions, many of the ideas from Section~\ref{sec_relpos} can also be
applied here. The main observation is that the search for an alternative match
can be realised by unifying a part of $\rheadB$ with $\rheadA$ in a way that resembles
our unification of $\rbodyB$ with $\rheadA$ in Section~\ref{sec_relpos}.

To realise this, we define a function \FuncSty{extend$^\square$} as a small modification of
Algorithm~\ref{alg_extend_positive}, where we simply replace
$\rbodyB$ in L\ref{alg:extend:for_i} and L\ref{alg:extend:mprime} by
$\rheadB$, and \FuncSty{check$^+$} in L\ref{alg:extend:check} by \FuncSty{check$^\square$},
which is defined in Algorithm~\ref{alg_check_restraint} and explained next.

\begin{algorithm}\caption{\FuncSty{check$^\square$}}\label{alg_check_restraint}

  \KwIn{rules $\arule_1,\arule_2$, atom mapping $m$ with mgu $\eta$}
  \KwOut{\textit{true} if a restraint is found for $m$}
  
  $\rhead^m_2 \leftarrow \dom(m)$\; \label{alg:check_restraint:defpsi21}
  $\rhead^\ell_2 \leftarrow \ghost{$\{\rheadB[j]\in(\rheadB{\setminus}\,\rhead^m_2)\mid j\,{<}\,\maxidx{m}\}$}$\; \label{alg:check_restraint:defpsi22left}
  $\rhead^r_2 \leftarrow \ghost{$\{\rheadB[j]\in(\rheadB{\setminus}\,\rhead^m_2)\mid j\,{>}\,\maxidx{m}\}$}$\; \label{alg:check_restraint:defpsi22right}
  
  \lIf{$x\eta\in\Nlang$ for some $x\in\Vlang_\forall$}{\label{alg:check_restraint:noUniversalNull}%
    \KwRet{false} \label{alg:check_restraint:noUniversalNull_return}
  }

  \lIf{$z\eta\in\Nlang$ for some $z\in\Vlang_\exists$ in $\rhead^\ell_2$}{\label{alg:check_restraint:noNullsPsi22_left}%
    \ghost{\KwRet{false}}
  }
  
  \If{$z\eta\in\Nlang$ for some $z\in\Vlang_\exists$ in $\rhead^r_2$}{\label{alg:check_restraint:noNullsPsi22}%
    \KwRet{\extendres{$\arule_1$,$\arule_2$,$m$}}\;
  }

  \If{$\rhead^m_2$ contains no existential variables}{\label{alg:check_restraint:containsExists}
    \KwRet{\extendres{$\arule_1$,$\arule_2$,$m$}}\;
  }

  $\tilde{\mathcal{I}}_a \leftarrow \rbodyB\eta_\forall\omega_\forall$\;\label{alg:check_restraint:assignIaTilde}
  \lIf{$\tilde{\mathcal{I}}_a \models \exists\vec{z}_2.\, \rheadB\eta_\forall \omega_\forall$}{\label{alg:check_restraint:modelsPsi2}
    \KwRet{false} \label{alg:check_restraint:modelsPsi2_return}
  }

  $\mathcal{I}_a \leftarrow \tilde{\mathcal{I}}_a \cup \rheadB\eta_\forall \omega$\;\label{alg:check_restraint:assignIa}
  $\tilde{\mathcal{I}}_b \leftarrow \mathcal{I}_a \cup (\rbodyA \cup \rhead^\ell_2\cup \rhead^r_2)\eta\omega$\label{alg:check_restraint:assignIbTilde}

  \lIf{$\tilde{\mathcal{I}}_b \models \exists\vec{z}_1.\, \rheadA\eta_\forall\omega_\forall$}{\label{alg:check_restraint:models}%
    \ghost{\KwRet{\extendres{$\arule_1$,$\arule_2$,$m$}}}
  }
  \lIf{$\rheadB\eta\omega \subseteq \tilde{\mathcal{I}}_b$}{\label{alg:check_restraint:given}%
    \KwRet{\extendres{$\arule_1$,$\arule_2$,$m$}}
  }

  \KwRet{true}\;\label{alg:check_restraint:returnTrue}
\end{algorithm}
We use the notation for $\arule_1$, $\arule_2$, $\omega$, $\Vlang_\exists$, and $\Vlang_\forall$
as introduced in Section~\ref{sec_relpos}, and again use atom mapping $m$ to represent our current
hypothesis for a possible match.
What is new now is that unified atoms in $\dom(m)$ can contain existentially quantified variables,
though existential variables in the range of $m$ (from $\rheadA$) are still replaced by nulls
as in Algorithm~\ref{alg_extend_positive}, L\ref{alg:extend:check}.
An existential variable in $\rheadB$ might therefore be unified with a constant, null, or universal variable of $\rheadA$.
In the last case, where we need a unifier $\eta$ with $z\eta = x\eta$ for $z\in\Vlang_\exists$ and $x\in\Vlang_\forall$,
we require that $x\eta=z\eta\in\Vlang_\forall$ so that $\eta$ only maps to variables
in $\Vlang_\forall$.
The unifier $\eta$ then simultaneously represents the matches $h_1$, $h_2$, and $h^A$ from Definition~\ref{def_restrained}.

\begin{example}
For rules $\arule_1= r(x,y) \to s(x,x,y)$ and
$\arule_2= a(z) \to \exists v.\, s(z,v,v)\wedge b(v)$, and mapping $m=\{s(z,v,v)\mapsto s(x,x,y)\}$,
we obtain a unifier $\eta$ that maps all variables to $x$ (we could also use $y$, but not the existential $v$).
Let $x\omega=c$ be the constant that $x$ is instantiated with.
Then we can apply $\arule_2$ to $\tilde{\Inter}_a = \{ a(z)\eta\omega \} = \{ a(c) \}$ with match
$h_2=\{z\mapsto c, v\mapsto n \}$ to get $\Inter_a=\tilde{\Inter}_a\cup\{ s(c,n,n),b(n)\}$,
and $\arule_1$ to $\tilde{\Inter}_b=\Inter_a\cup\{r(c,c),b(c)\}$ with match $h_1=\{x\mapsto c, y\mapsto c\}$
to get $\Inter_b=\tilde{\Inter}_b\cup\{s(c,c,c)\}$. Note that we had to add $b(c)$ to 
obtain the required alternative match $h^A$, which maps $n$ to $v\eta\omega=c$ and $c$ to itself.
\end{example}

As in the example, a most general unifier $\eta$ yields a candidate $h^A$ that 
maps every null of the form $v\omega_\exists$ to $v\eta_\exists\omega_\forall$.
Likewise, for $i\in\{1,2\}$, $h_i=\eta_\forall\omega$ are the (extended) matches,
while $\eta_\forall\omega_\forall$ are the body matches.
The image of the instantiated $\rheadB\eta_\forall\omega$ under the alternative match $h^A$
is given by $\rheadB\eta\omega$.
The corresponding interpretations are
$\Inter_a= \rbodyB\eta_\forall\omega_\forall \cup \rheadB\eta_\forall\omega$
and
$\Inter_b= \Inter_a\cup \rbodyA\eta_\forall\omega_\forall \cup \rheadA\eta_\forall\omega \cup (\rhead\setminus\dom(m))\eta\omega$,
where $(\rheadB\setminus\dom(m))\eta\omega$ provides additional atoms required for the alternative match but
not in the mapped atoms of $\rheadB$.
With these intuitions, Algorithm~\ref{alg_check_restraint} can already be understood.

It remains to explain the conditions that are checked before returning \textit{true}.
As before, we partition $\dom(m)$ into mapped atoms $\rhead^m_2$ and left and right remainder atoms.
Checks in L\ref{alg:check_restraint:noUniversalNull}--L\ref{alg:check_restraint:noNullsPsi22}
ensure that the only variables mapped by $\eta$ to nulls (necessarily from $\rheadA\omega_\exists$)
are existential variables in $\rhead^m_2$: such mappings are possible by $h^A$.
Extending $m$ further is only promising if the nulls only stem from atoms in $\rhead^r_2$.

Check L\ref{alg:check_restraint:containsExists} continues the search when no atoms with existentials
have been selected yet. Selecting other atoms first might be necessary by our order, but no
alternative matches can exist for such mappings (yet).
Lines L\ref{alg:check_restraint:modelsPsi2} and L\ref{alg:check_restraint:models} check that
the matches $h_1$ and $h_2$ are indeed unsatisfied.
Extending $m$ might fix L\ref{alg:check_restraint:modelsPsi2} by making $\tilde{\Inter}_a$ smaller,
whereas L\ref{alg:check_restraint:models} cannot be fixed.
Finally, L\ref{alg:check_restraint:given} ensures condition \eqref{it_rest_noalt} of Definition~\ref{def_restrained}.

\begin{example}
  Consider rules $\arule_1 = b(x,y) \rightarrow r(x,y,x,y) \wedge q(x,y)$, $\arule_2 = a(u,v) \rightarrow \exists w.\, r(u,v,w,w) \wedge r(v,u,w,w)$, and mapping $m = \{ r(u,v,w,w)\mapsto r(x,y,x,y) \}$.
  We obtain unifier $\eta$ mapping all variables to a single universally quantified variable, say $x$.
  We reach $\tilde\Inter_b = \{ a(c,c), r(c,c,n,n), b(c,c), r(c,c,c,c) \}$, based on $\tilde\Inter_a = \{ a(c,c) \}$ ($x\omega = c$), for which $\arule_1$ is applicable but $h^A = \{ n\mapsto c, c\mapsto c \}$ is already an alternative match on $\tilde\Inter_b$, recognized by L\ref{alg:check_restraint:given}.
\end{example}
\begin{restatable}{theorem}{thmRelResCorrect}\label{theo_relres_correct}
For rules $\arule_1$ and $\arule_2$ that (w.l.o.g.) do not share variables,
$\arule_1\relres\arule_2$ holds according to Definition~\ref{def_restrained} for some $\Inter_a\neq\Inter_b$ iff \extendres{$\arule_1$,$\arule_2$,$\emptyset$}${}=\textit{true}$.
\end{restatable}

The case $\Inter_a=\Inter_b$, which Theorem~\ref{theo_relres_correct} leaves out, is possible
\cite[Example~5]{KR20:cores}, but requires a slightly different algorithm.
We can adapt Algorithm~\ref{alg_check_restraint} by restricting to one rule, for which we map
from atoms in $\rhead$ to atoms in $\rhead\omega_\exists$. The checks (for $\rheadB$) of 
Algorithm~\ref{alg_check_restraint} remain as before, but we only need to compute a single
$\Inter$ that plays the role of $\Inter_a$ and $\Inter_b$.
Check L\ref{alg:check_restraint:given} is replaced by a new check\\[1ex]
\mbox{}~~~\lIf{$\rhead\,\eta_\exists = \rhead\,\omega_\exists$}{%
      \KwRet{false}
    }\smallskip

\noindent
to ensure that at least one null is mapped differently in the alternative match. With these modifications,
we can show an analogous result to Theorem~\ref{theo_relres_correct} for the case $\Inter_a=\Inter_b$.

\section{Implementation and Global Optimisations}\label{sec_impl}

We provide a C++ implementation of our algorithms, which also includes
some additional optimisations and methods as described next.
Our prototype is integrated with the free rule engine VLog~\cite{UJK:VLog2016},
so that we can use its facilities for loading rules and checking MFA (see Section~\ref{sec_eval}).
Reasoning algorithms of VLog are not used in our code.

The algorithms of Sections~\ref{sec_relpos} and \ref{sec_relres} 
use optimisations that are \emph{local} to the task of computing
dependencies for a single pair of rules.
The quadratic number of potential rule pairs is often so large, however,
that even the most optimised checks lead to significant overhead.
We therefore build index structures that map predicates $p$ to rules that 
use $p$ in their body or head, respectively.
For each rule $\arule_1$, we then check $\arule_1\relpos\arule_2$ only
for rules $\arule_2$ that mention some predicate from $\rhead(\arule_1)$
in their body, and analogously for $\arule_1\relres\arule_2$.

Specifically for large rule sets, we further observed that many rules
share the exact same structure up to some renaming of predicates and variables.
For every rule pair considered, we therefore create an abstraction that
captures the co-occurrence of predicates but not the concrete predicate names.
This abstraction is used as a key to cache results of prior computations
that can be re-used when encountering rule pairs with the exact same pattern
of predicate names.

Besides these additional optimisations, we also implemented
unoptimised variants of the algorithms of Sections~\ref{sec_relpos} and \ref{sec_relres}
to be used as a base-line in experiments.
Instead of our goal-directed check-and-extend strategy, we simply iterate over all possible
mappings until a dependency is found or the search is completed.

\section{Evaluation}\label{sec_eval}

We have evaluated our implementation regarding (1) efficiency of our optimisations and
(2) utility for solving practical problems.
The latter also led to the first study of so-called \emph{core stratified} real-world rule sets. 
Our evaluation machine is a mid-end server (Debian Linux 9.13; Intel Xeon CPU E5-2637v4@3.50GHz;
384GB RAM DDR4; 960GB SSD), but our implementation is single-threaded and 
did not use more than 2GB of RAM per individual experiments.
\begin{tr}
Our source code, experimental data, instructions for repeating all experiemtns, and our own raw measurements are available on \href{https://github.com/knowsys/2022-ISWC-reliances}{GitHub}.
\end{tr}

\subsubsection*{Experimental Data}
All experiments use the same corpus of rule sets, created from real-world OWL ontologies
of the \emph{Oxford Ontology Repository} (\url{http://www.cs.ox.ac.uk/isg/ontologies/}). %
OWL is based on a fragment of first-order logic that overlaps with existential rules.
OWL axioms that involve datatypes were deleted; any other axiom was syntactically
transformed to obtain a Horn clause that can be
written as a rule. This may fail if axioms use unsupported features,
especially those related to (positive) disjunctions and equality.
We dropped ontologies that could not fully be translated or that required no existential
quantifier in the translation.

Thereby \numOfRuleSets of the overall \numOfOntologies ontologies were converted
to existential rules, corresponding largely to those ontologies in the logic
Horn-$\mathcal{SRI}$ \cite{KRH:HornDLs2012}. The corpus contains
63 small (18--1,000 rules), 90 medium (1,000--10,000 rules), and 48 large (10,000--167,351 rules) sets.
Our translation avoided normalisation and auxiliary predicates, which would
profoundly affect dependencies. This also led to larger rule bodies and heads, both
ranging up to 31 atoms.
\subsubsection*{Optimisation impact}
We compare four software variants to evaluate the utility of our proposed optimisations.
Our baseline \imp[N] is the unoptimised version described in Section~\ref{sec_impl},
while \imp[L] uses the locally optimised algorithms of Sections~\ref{sec_relpos} and \ref{sec_relres}.
Version \imp[G] is obtained from \imp[N] by enabling the global optimisations of Section~\ref{sec_impl},
and \imp[A] combines all optimisations of \imp[L] and \imp[G].
For each of the four cases, we measured the total time of determining all positive reliances and all
restraints for each rule set. A timeout of 60sec was used.
The number of timeouts for each experiment was as follows:

\begin{center}\mbox{}\hfill
\begin{tabular}{l@{~~~~}c@{~~~~}c@{~~~~}c@{~~~~}c}
$\relpos$ & \imp[N] & \imp[L] & \imp[G] &\imp[A] \\
 & 80 & 48 & 24 &3\\
\end{tabular}\hfill
\begin{tabular}{l@{~~~~}c@{~~~~}c@{~~~~}c@{~~~~}c}
$\relres$ & \imp[N] & \imp[L] & \imp[G] &\imp[A] \\
 & 87 & 52 & 35 & 15\\
\end{tabular}\hfill\mbox{}
\end{center}

\begin{table}[tbp]
  \caption{Number of rule sets achieving a given order of magnitude of speed-up for computing $\relpos$ (left) and $\relres$ (right)
    from one variant to another; t.o. gives the number of avoided timeouts}  \label{tab_opt}

  \newcommand{\myspace}{\hspace{2pt}}
  \newcommand{\smallheader}[1]{\ghost{\rotatebox{45}{#1}}\phantom{Xxi}\rule{0pt}{5ex}}
  \newcommand{\smallerheader}[1]{\ghost{\rotatebox{45}{#1}}\phantom{Xx}\rule{0pt}{5ex}}
   \centering
  \begin{tabular}{@{\myspace}l@{\myspace}}
	\rule{0pt}{5ex}\\
    \midrule
    $\imp/\imp[L]$\\
    $\imp[G]/\imp[A]$\\
    $\imp[N]/\imp[G]$\\
    $\imp[L]/\imp[A]$\\
    \bottomrule
  \end{tabular}~
  \begin{tabular}{@{}r@{\myspace\myspace}r@{\myspace}r@{\myspace}r@{\myspace}r@{\myspace\myspace\myspace}r@{\myspace}}
      \smallheader{${=}1$} & \smallheader{${<}10$} & \smallheader{${<}10^2$} & \smallheader{${<}10^3$} & \smallheader{${\geq}10^3$} & \smallerheader{t.o.} \\
    \midrule
    48	& 104 &	14 &  1 &  2 & 32 \\
    103	&  67 &	 9 &  1 &  0 & 21 \\
    24  &   1 & 27 & 33 & 60 & 56 \\
    5	  &  33 & 30 & 41 & 47 & 45 \\
    \bottomrule
  \end{tabular}~
  \begin{tabular}{@{}r@{\myspace\myspace}r@{\myspace}r@{\myspace}r@{\myspace}r@{\myspace\myspace\myspace}r@{\myspace\myspace}}
    \smallheader{${=}1$} & \smallheader{${<}10$} & \smallheader{${<}10^2$} & \smallheader{${<}10^3$} & \smallheader{${\geq}10^3$} & \smallerheader{t.o.} \\
    \midrule
     53	& 92 & 17 &	 2 &	2 &	35 \\
     90 &	81 &	9 &	 1 &	0 &	20 \\
     35 &	11 & 53 &	30 & 20 &	52 \\
     17 &	72 & 48 &	10 & 17 &	37 \\
    \bottomrule
  \end{tabular}
\end{table}

To present the remaining results, we focus on \emph{speed-up}, i.e., the ratio of runtime of a less optimised variant over
runtime of a more optimised one.
Table~\ref{tab_opt} classifies the observed speed-ups in several scenarios by their order of magnitude.
For example, in the left table, the number 14 in line $\imp/\imp[L]$ and column ``${<}10^2$'' means
that for 14 of the 201 rule sets, \imp[L] was between $10$--$10^2$ times faster than \imp[N].
Note that $\imp[G]/\imp[A]$ shows the effect of adding \emph{local} optimisations to \imp[G].
Column ``${=}1$'' shows cases where both variants agree,
and column ``t.o.'' cases where the optimisation avoided a prior timeout (the speed-up cannot be computed
since the timeout does not correspond to a time).

We conclude that both \imp[L] and \imp[G] can lead to significant performance gains across a range of ontologies.
Strong effects are seen against the baseline ($\imp/\imp[L]$ and $\imp/\imp[G]$), but also (to a slightly lesser
extent) against variants with the other optimisations ($\imp[G]/\imp[A]$ and $\imp[L]/\imp[A]$).
Overall, $\relres$ turned out to be slower than $\relpos$, with the global optimisations being less effective.

\begin{figure}[tbp]
  \mbox{}\hfill\includegraphics[width=.92\linewidth]{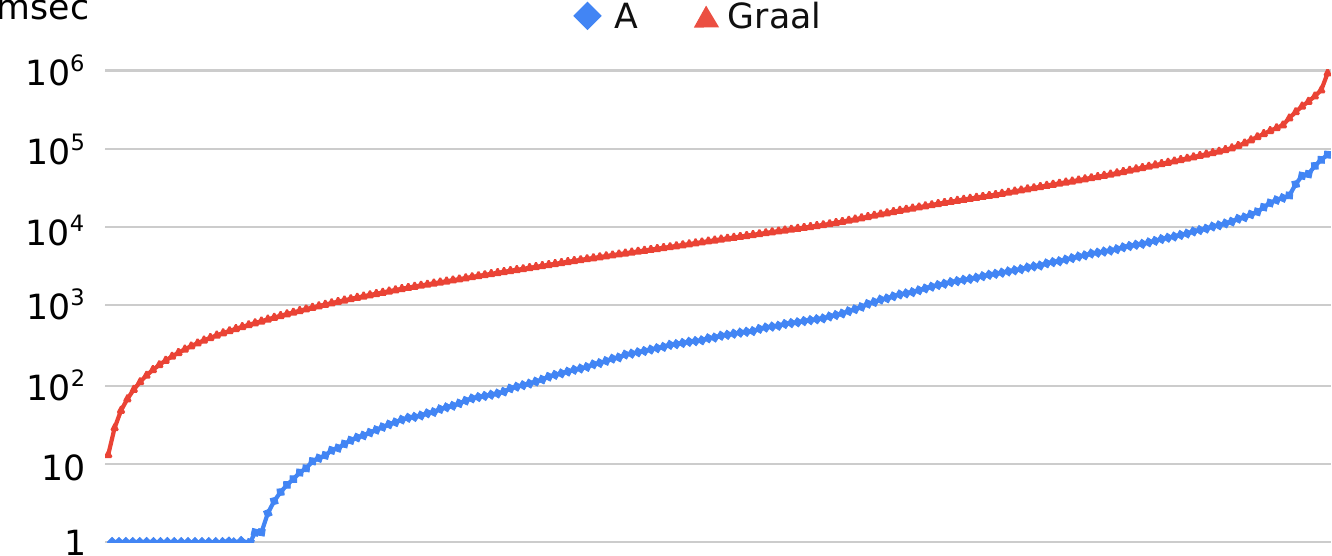}\hfill\mbox{}
  \caption{Positive reliance computation in Graal (top) and our system (bottom)}
  \label{fig:graal}
\end{figure}
\subsubsection*{Acyclic positive reliances} 
For rule sets where the graph of positive reliances is acyclic, query answering is possible with many existing rule engines~\cite{BLMS11:decline}.
To evaluate how our work compares to the state of the art in computing this graph, we
measure the time taken by Graal to find all positive reliances and compare them to our prototype \imp[A] from above.
The results are shown in Figure~\ref{fig:graal}.

Our approach consistently outperformed Graal by about one order of magnitude.
Overall, we can classify 178 ontologies in under 1sec, making this analysis feasible at reasoning time.
The difference in execution time is explained by our optimisations: given two rules $\rho_1$ and $\rho_2$, Graal computes all (exponentially many in the worst case) different ways to unify the $\rhead(\rho_1)$ with $\rbody(\rho_2)$ while our implementation (1) stops when a positive reliance is discovered, (2) discards atom mappings when a negative result is guaranteed, and (3) caches results of previous computations.

Recall that Graal uses a slightly weaker notion of positive reliance (cf.\ Sect.~\ref{sec_deps}),
which leads to more cycles: we find 36 acyclic sets in Graal, but 70 such sets in our system.

\subsubsection*{Faster MFA}
\emph{Model-faithful acyclicity} (MFA) is an advanced analysis of rule sets that can discover
decidability of query answering in many cases, but is \TwoExpTime-complete \cite{CG+13:acyclicity}.
However, instead of performing this costly analysis on the whole rule set, an equivalent result 
can be obtained by analysing each strongly connected components of the $\relpos$-graph individually.
We measure the times for both approaches using the 
MFA implementation of VLog and our optimised variant \imp[A], with a timeout of 30min per rule set.
The two variants are denoted \imp[V] (VLog MFA) and \imp[C] (component-wise MFA).

Using \imp[C], 163 ontologies are classified as MFA, 33 fail MFA, and 5 cases time out.
\imp[V] times out in 10 cases, but agrees on all other outcomes.
\imp[C] is slower in three cases that still run in under 50msec.
The numbers of speed-ups, grouped by order of magnitude, are as follows:

\begin{center}
\newcommand{\smallheader}[1]{#1}%
  \begin{tabular}{l@{~~~~}r@{~~~~}r@{~~~~}r@{~~~~}r@{~~~~}r}
Speed-up & \smallheader{$=1$} & \smallheader{$<10$} & \smallheader{$<10^2$} & \smallheader{$<10^3$} & \smallheader{$\geq 10^3$} \\
\midrule
$\imp[V]/\imp[C]$ & 0 & 85 & 54\phantom{$^3$} & 41\phantom{$^3$} & 11\phantom{$^3$} \\
  \end{tabular}%
\end{center}

We conclude that our optimised reliance computation is a feasible approach for 
speeding up MFA analysis.

\subsubsection*{Core stratification}
We can use our implementation to determine how common this favourable property (cf.\xspace Sect.~\ref{sec_deps}) is among real-world ontologies.
The analysis was feasible for 200 rule sets in our corpus,
yielding 44 core stratified sets with up to 121,712 rules.
One can improve this result by considering \emph{pieces},
minimal subsets of rule heads where each two atoms refer to a common existentially
quantified variable \cite{BLMS11:decline}. Each rule can then equivalently be replaced
by several rules, each combining the original body with one of the pieces of the original head.
Applying this transformation to our rule sets leads to more fine-grained dependencies
that have fewer cycles over $\relres$. With this modification, 75 rule sets are core stratified.

Our implementation fails in one case (ontology ID 00477), containing 167,351 rules like
$A(x)\to\exists v.\, \textsf{located-in}(x,v)\wedge B(v)$, for various $A$ and $B$.
The required $>28\times 10^9$ checks, though mostly cached, take very long.
In spite of many $\relres$-relations, the set is core-stratified as it describes a proper meronomy.

The remaining 125 rule sets are not core stratified.
To validate the outcome, we have analysed these sets manually, and found several
common reasons why ontologies were indeed not core stratified (and therefore correctly classified in our implementation).
The following two examples explain two typical situations.

\begin{example}
In some cases, core stratification fails even though there is a natural rule application order that always leads
to a core. Consider the rules
$\arule_1= a(x) \rightarrow \exists v.\, r(x, v) \wedge b(v)$,
$\arule_2= r(x, y) \rightarrow s(y, x)$, and
$\arule_3= s(x, y) \rightarrow r(y, x)$.
This set is not core stratified since we have $\arule_1 \relpos \arule_3$, $\arule_2 \relpos \arule_3$, $\arule_3 \relpos \arule_2$, and $\arule_3 \relres \arule_1$.
However, prioritising $\arule_2$ and $\arule_3$ over $\arule_1$
(i.e., using a \emph{Datalog-first} strategy~\cite{Carral+19:ChasingSets}) always leads to a core.
Indeed, the positive reliance $\arule_1 \relpos \arule_3$ over-estimates relevant rule applications,
since no new atom produced by $\arule_1$
can (indirectly) lead to an application of $\arule_3$.
\end{example}

\begin{example}\label{ex_noncore_trans}
In other cases, there is indeed no data-independent strategy for rule applications that would always 
lead to a core.
Consider the rules $\rho_1= a(x) \rightarrow \exists v.\, r(x, v) \wedge b(v)$
and $\rho_2= r(x, y) \wedge r(y, z) \rightarrow r(x, z)$.
Both are common in OWL ontologies with existential axioms and transitive roles.
The rule set is not core stratified since $\rho_1\relpos\rho_2$ and $\rho_2\relres\rho_1$.

Consider $\Inter_a = \{a(1), a(2), r(1, 2)\}$.
Applying $\arule_1$ over $\Inter_a$ to all matches 
yields $\Inter_b = \Inter_a \cup \{r(1, n), b(n), r(2, m), b(m)\}$,
which makes $\arule_2$ applicable to obtain $\Inter_c = \Inter_b \cup \{r(1, m)\}$.
Here we have the alternative match $h^A = \{1 \mapsto 1, 2 \mapsto 2, n \mapsto m\}$.

In contrast, applying $\arule_1$ only for the match $\{x\mapsto 2\}$ produces
$\Inter_b' = \Inter_a \cup \{r(2,n),b(n)\}$.
A subsequent application of $\rho_2$ yields $\Inter_c' = \Inter_b' \cup \{r(1,n)\}$, which is a core model.
Indeed, core models could often be achieved in such settings, but require fine-grained, data-dependent
strategies that cannot be found by static analysis (concretely: we could consider $r$ as a pre-order and
apply $\arule_1$ to the $r$-greatest elements first, followed by an exhaustive application of $\arule_2$).
\end{example}

Overall, our manual inspection supported the correctness of our computation and led to interesting
first insights about core stratification in practical cases. Regarding the contribution of this work,
our main conclusion of this evaluation is
that our proposed algorithms are able to solve real-world tasks that require the computation of positive reliances and
restraints over large ontologies.

\section{Conclusions}\label{sec_conc}

We have shown that even the complex forms of dependencies that arise
with existential rules can be implemented efficiently, and that doing so
enables a number of uses of practical and theoretical interest.
In particular, several previously proposed approaches can be made significantly faster or
implemented for the first time at all.
Our methods can be adapted to cover further cases, especially the
\emph{negative reliances}.

Our work opens up a path towards further uses of reliance-based analyses in practice.
Already our experiments on core stratification -- though primarily intended to
evaluate the practical feasibility of our restraint algorithm -- also showed that
(a) core stratification does occur in many non-trivial real-world ontologies, whereas
(b) there are also relevant cases where this criterion fails although a rule-based core computation seems to be within reach.
This could be a starting point for refining this notion.
It is also interesting to ask whether good ontology design should, in principle, lead to
specifications that naturally produce cores, i.e., that robustly avoid redundancies.
A different research path is to ask how knowledge of dependencies
can be used to speed up reasoning.
Indeed, dependencies embody characteristics of
existential rule reasoning that are not found in other rule languages,
and that therefore deserve further attention.

\begin{paper}
\paragraph{Supplemental Material Statement}
We provide full proofs in the techical report published on arXiv~\cite{Ivliev+22:Reliances-TR-2022}.
Our source code, experimental data, instructions for repeating all experiemtns, and our own raw measurements are available on \href{https://github.com/knowsys/2022-ISWC-reliances}{GitHub}.
\end{paper}

\paragraph{Acknowledgments}
This work is partly supported
by Deutsche Forschungsgemeinschaft (DFG, German Research Foundation) in project 389792660 (TRR 248, \href{https://www.perspicuous-computing.science/}{Center for Perspicuous Systems}),
by the Bundesministerium für Bildung und Forschung (BMBF, Federal Ministry of Education and Research) under European ITEA project 01IS21084 (\href{https://www.innosale.eu/}{InnoSale, Innovating Sales and Planning of Complex Industrial Products Exploiting Artificial Intelligence})
and \href{https://www.scads.de}{Center for Scalable Data Analytics and Artificial Intelligence} (ScaDS.AI),
by BMBF and DAAD (German Academic Exchange Service) in project 57616814 (SECAI, School of Embedded and Composite AI),
and by the \href{https://cfaed.tu-dresden.de}{Center for Advancing Electronics Dresden} (cfaed).

\bibliographystyle{splncs04}

\begin{tr}
\appendix
\section{Proof of Theorem~\ref{theo_relpos_correct}}

In the following, 
we may understand a substitution $\sigma$ 
as a homomorphism and therefore
as a match for some rule $\arule$ 
if the restriction of $\sigma$ to the constants and variables occurring in the body of $\arule$ 
is a homomorphism or match in the defined sense.
Furthermore,
for any function $f$ we define $\im(f) = f(\dom(f))$ as the image of $f$.
We make use of function compositions $g\circ f\colon A \to C$ for functions $f\colon A\to B$ and $g\colon B\to C$, defined by $(g\circ f)(x) = g(f(x))$ for all $x\in A$.
Recall that, in contrast, concatenation of substitutions $\eta$ and $\omega$ 
(which are also functions) is denoted by $\eta\omega$ and is defined by $(\eta\omega)(x)=\omega(\eta(x))$.

The next lemma establishes a connection between 
the unifier $\eta\omega$ in Algorithms \ref{alg_check_positive} and \ref{alg_check_restraint}
and any homomorphism $h$ that serves a witness for a reliance 
while also being a unifier for the considered atom mapping.
\begin{lemma}\label{Lemma_Tau}
  Let $m$ be an atom mapping 
  and $\eta$ the most general unifier for $m$
  with $\im(\eta) \cap \im(\omega) = \emptyset$.
  Let $h$ be a homomorphism
  which is also a unifier of $m$.
  Then there exists a function $\tau \colon \Clang \cup \Nlang \to \Clang \cup \Nlang$
  such that $h \subseteq \tau \circ (\eta\omega)$.
\end{lemma}
\begin{proof}
  We set (a) $\tau((\eta\omega)(x)) = h(x)$
  for every $x \in \dom(\eta)$
  and (b) $\tau(c) = c$ for every $c \in (\Clang \cup \Nlang) \setminus \im(\eta\omega)$.
  Recall that homomorphisms, like $h$, map from sets of atoms to instances (cf. Sect.~\ref{sec_prelims}), meaning that $\im(h)\subseteq \Clang \cup \Nlang$.
  Then, $h \subseteq \tau \circ (\eta\omega)$ by definition of $\tau$.
  
  We now need to argue that 
  $\tau$ is a function.
  $\tau$ is defined on all $x\in\Clang \cup \Nlang$ because of (a) and (b).
  Assume that $(\eta\omega)(x) = (\eta\omega)(y) = t$
  for some $x, y \in \dom(\eta)$.
  If $\eta(x) = t$,
  then $\eta(y) = t$,
  and vice versa,
  since the images of $\eta$ and $\omega$ 
  are disjoint.
  In this case we conclude that
  $h(x) = h(y) = t$ as well
  because $\eta$ is the most general unifier.
  Assume now that $\eta(x) \neq t$ and $\eta(y) \neq t$.
  Since $\omega$ assigns unique constants or nulls 
  to every variable,
  we have that $\eta(x) = \eta(y)$.
  But then $h(x) = h(y)$ again follows from $\eta$
  being the most general unifier.\qed
\end{proof}

In the following lemma, we have two atom sets $\mathcal A$ and $\mathcal B$.
As we assumed in sections \ref{sec_relpos} and \ref{sec_relres}, the variables occurring in $\mathcal A$ or $\mathcal B$ 
divide into variables from sets $\Vlang_\exists$ and $\Vlang_\forall$.
Notice that, 
for a substitution $\theta$, $\theta_\forall(v)=v$ for all $v\in\Vlang_\exists$ and $\theta_\exists(v)=v$ for all $v\in\Vlang_\forall$.
This lemma plays a key role in proving completeness of the 
reliance algorithm.

\begin{lemma}\label{Lemma_Completeness}
  Let $\mathcal{A}$ and $\mathcal{B}$ be two sets of atoms,
  $\mathcal{I}$ an interpretation,
  and $h$ a homomorphism from $\mathcal{A}$ to $\mathcal{I}$.
  Let $\eta$ be a substitution
  such that $h \subseteq \tau \circ (\eta\omega)$
  for some $\tau \colon \Clang \cup \Nlang \to \Clang \cup \Nlang$
  with $\tau(c) = c$ for every $c \in \Clang \cup \Nlang$
  occurring in $\mathcal{A}$ and $\mathcal{B}$.  
  If there is no homomorphism $h^\star$ 
  from $\mathcal{B}$ to $\mathcal{I}$
  that agrees with $h$ on all universal variables,
  then $\mathcal{A}\eta\omega \not\models \exists \vec z .\ \mathcal{B}\eta_\forall \omega_\forall$ 
  where $\vec z$ are all (existential) variables occurring in $\mathcal{B}\eta_\forall\omega_\forall$.
\end{lemma}
\begin{proof}
  Assume for a contradiction that $\mathcal{A}\eta\omega \models \exists\vec z .\ \mathcal{B}\eta_\forall \omega_\forall$.
  Then there is a substitution 
  $\eta'_\exists$
  mapping the existential variables in $\mathcal{B}$ (i.e., $\vec z$)
  such that $\mathcal{B}\eta'_\exists \eta_\forall \omega_\forall \subseteq \mathcal{A}\eta\omega$.
  Note that $h_\forall \subseteq \tau \circ (\eta_\forall\omega_\forall)$.
  We define $h^\star = (\tau \circ \eta'_\exists) h_\forall$.
  Therefore, $h^\star$ agrees with $h$ on all universal variables.
  
  Starting with $\mathcal{B}\eta'_\exists \eta_\forall \omega_\forall \subseteq \mathcal{A}\eta\omega$,
  we can concatenate $\tau$ on both sides to obtain
  $\mathcal{B}(\tau \circ \eta'_\exists)(\tau \circ (\eta_\forall \omega_\forall))
  = \mathcal{B} h^\star
   \subseteq \mathcal{A}(\tau \circ (\eta\omega)) = \mathcal{A}h$
  and hence $\mathcal{B} h^\star \subseteq \mathcal{I}$.
  But this implies that $h^\star$ 
  is a homomorphism from $\mathcal{B}$ to 
  $\mathcal{I}$ that agrees with $h$ on all universal variables.
  The first step requires that $\tau$ does not change any of the constants 
  or nulls occurring in $\mathcal{A}$ and $\mathcal{B}$.\qed
\end{proof}

\thmRelPosCorrect*

\begin{proof}
    We separate the correctness of Algorithm~\ref{alg_check_positive} into soundness and completeness.
    \paragraph{Soundness:}
    The call to \extendpos{$\arule_1,\arule_2,\emptyset$}
    returns \textit{true} iff \checkpos{$\arule_1, \arule_2, m, \eta$}${}=\textit{true}$
    for some atom mapping $m$ and some mgu $\eta$,
    which means that L\ref{alg:check_positive:return_true} is reached.
    We set $\rbody^m_2$, $\rbody^\ell_2$ and $\rbody^r_2$ 
    as in Algorithm~\ref{alg_check_positive}.
    Furthermore, we define $\rhead^m_1 = \im(m)$.

    L\ref{alg:check_positive:Ia}
    constructs the interpretation $\Inter_a = (\rbody_1 \cup \rbody^\ell_2 \cup \rbody^r_2)\eta\omega$.
    From this, 
    we can immediately conclude that $\eta\omega$ 
    is a match for $\rho_1$ over $\mathcal{I}_a$.
    It is unsatisfied,
    because of the check in L\ref{alg:check_positive:modelsPsi1}.
    Therefore $\rho_1$ is applicable with the unsatisfied 
    match for $\eta\omega$ over $\mathcal{I}_a$.
    We define $\mathcal{I}_b$
    as the result of applying this match,
    extending existential variables in $\rho_1$ 
    with their image in $\omega_\exists$ (as constructed in L\ref{alg:check_positive:Ib}).
    Note that  $\Inter_a$ cannot 
    contain any null introduced by the above application 
    because of the checks in L\ref{alg:check_positive:noNulls1},
    L\ref{alg:check_positive:noNulls22_left}, and L\ref{alg:check_positive:noNulls22}.
    From the check in L\ref{alg:check_positive:modelsPsi2},
    we know that $\rho_2$ is applicable over $\Inter_b$.
    Thus, we have $\Inter_a \subseteq \Inter_b$ and a function $\eta\omega$
    satisfying conditions (a) and (b) of Definition~\ref{def_posrel}.
    Condition (c) is satisfied because of the check in L\ref{alg:check_positive:modelsPhi2}.

    \paragraph{Completeness:}
    To prove completeness,
    we assume $\arule_1\relpos\arule_2$.
    Hence, there are interpretations $\mathcal{J}_a \subseteq \mathcal{J}_b$
    and functions $h_1$ and $h_2$
    that satisfy Definition~\ref{def_posrel}.
    We may assume, w.l.o.g., that $h_1'$ and $h_2'$ 
    map every existential variable $v$ in their domain to $\omega_\exists(v)$.
    Since $h_2$ is a match for $\rho_2$ over $\mathcal{J}_b$
    but not over $\mathcal{J}_a$,
    there must be a partition 
    $\rbody_2 = B^m_2 \mathrel{\dot\cup} \bar{B}^m_2$
    and $\rhead_1 = H^m_1 \mathrel{\dot\cup} \bar{H}^m_1$
    such that $h_2(B^m_2) = h_1'(H^m_1)$
    and $h_2(\bar{B}^m_2) \subseteq \mathcal{J}_a$.
    We define $h\colon \Clang \cup \Nlang \cup \Vlang \to \Clang \cup \Nlang \cup \Vlang$ as
    $$h(x) = \left\{ \begin{array}{ccl}
      h_2(x) && \text{if $x$ is a variable in } \arule_2 \\
      h_1'(x) && \text{if $x$ is a variable in } \arule_1 \\
      x && \text{otherwise}. \\
    \end{array}
    \right.$$
    The above function is well-defined because $\rho_1$ and $\rho_2$
    are presumed to not share any variables.
    By definition of $h$
    we have that $B^m_2 h = H^m_1 h$,
    implying that $B^m_2$ and $H^m_1$ are unifiable
    and that there is an atom mapping $m$ with $\dom(m) = B^m_2$
    and $\im(m) = H^m_1$
    such that $h$ is a unifier of $m$.
    Therefore, 
    there is also a most general unifier $\eta$ of $m$.
    Since $\omega$ is assumed to assign terms only to constants 
    and nulls not contained in $\arule_1$ or $\arule_2$,
    we can conclude that $\im(\eta) \cap \im(\omega) = \emptyset$,
    and, by Lemma~\ref{Lemma_Tau},
    that $h \subseteq \tau \circ (\eta\omega)$
    for some $\tau \colon \Clang \cup \Nlang \to \Clang \cup \Nlang$.

    In the following,
    we show that each if-condition in Algorithm~\ref{alg_check_positive}
    when called on $m$ and $\eta$ fails,
    which implies that \textit{true} is returned.
    Note that $\bar{B}^m_2 = \rbody^\ell_2 \cup \rbody^r_2$.

    Any variable assigned to a null by $\eta$
    must also be assigned to the same null in $h$,
    since $\eta$ is more general than $h$.
    But then $\mathcal{J}_a$ would need 
    to contain a null introduced by the application of $\rho_2$.
    This follows because 
    $h(\rbody_1) = h_1(\rbody_1) \subseteq \mathcal{J}_a$ and
    $h(\bar{B}^m_2) = h_2(\bar{B}^m_2) \subseteq \mathcal{J}_a$.

    We handle the remaining checks with Lemma~\ref{Lemma_Completeness}.
    Note that since $\eta$ does not assign 
    any existential variables $\eta = \eta_\forall$.
    For L\ref{alg:check_positive:modelsPsi1},
    we set $\mathcal{A}_1 = \rbody_1 \cup \bar{B}^m_2$,
    $\mathcal{B}_1 = \rhead_1$ and $\mathcal{I}_1 = \mathcal{J}_a$.
    Then we have $\Inter_a = \mathcal{A}_1\eta\omega$
    and by Lemma~\ref{Lemma_Completeness}
    that $\Inter_a \not\models \exists \vec z.\ \rhead_1\eta_\forall\omega_\forall$.
    Hence, the check on L\ref{alg:check_positive:modelsPsi1} fails.
    For L\ref{alg:check_positive:modelsPhi2}
    we set $\mathcal{A}_2 = \mathcal{A}_1$, 
    $\mathcal{B}_2 = \rbody_2$ and $\mathcal{I}_2 = \mathcal{J}_a$.
    Note here that $\rbody_2 \omega\eta \subseteq \mathcal{I}_a$
    is equivalent to stating $\mathcal{I}_a \models \exists z.\, \rbody_2\eta_\forall\omega_\forall$.
    For L\ref{alg:check_positive:modelsPsi2}
    we have $\mathcal{A}_3 = \mathcal{A}_2 \cup \rhead_1\eta\omega$, 
    $\mathcal{B} = \rhead_2$ and $\mathcal{I}_3 = \mathcal{J}_b$.
    
    It remains to be shown that the iteration 
    in function \extendpos{$\arule_1,\arule_2,\emptyset$}
    eventually reaches the postulated mapping $m$
    or terminates with result \textit{true} before.
    Recall that the overall procedure only stops and returns
    \textit{false} if all atoms from $\rbody_2$ 
    have been tried to be the initial mapping.
    Hence, \textit{false} cannot be returned
    before either $m$ is reached (in which case it must return \textit{true}
    as shown above)
    or some non-empty subset $m'$ of $m$ is considered. 

    Because there is a unifier $\eta$ for $m$,
    there is one for every non-empty subset $m'\subseteq m$.
    As the order in which atom mappings are created
    depends on the assumed order of the atoms (in rule bodies and heads),
    we need to show that if any such mapping $m'$ with mgu $\eta'$ is reached,
    it is not rejected by the call of \checkpos{$\arule_1,\arule_2,m',\eta'$}.
    As $\eta'$ is an mgu and $m'\subseteq m$,
    we have $\eta \subseteq \tau' \circ (\eta'\omega)$
    for some $\tau' \colon \Clang \cup \Nlang \to \Clang \cup \Nlang$
    by Lemma~\ref{Lemma_Tau}.
    \begin{description}
    \item[L\ref{alg:check_positive:noNulls1}:] 
        Every variable that is assigned to a null by $\eta'$ must also 
        be assigned to the same null in $\eta$,
        since the latter is more general.
        But this is not possible because this check failed for 
        \checkpos{$\arule_1,\arule_2,m',\eta$}.
    \item[L\ref{alg:check_positive:noNulls22_left}:] 
        Because of the fixed order of atoms, it holds that $\rbodyB^l$ obtained in the iteration with mapping $m'$ is a subset of $\rbodyB^l$ obtained in the iteration with atom mapping $m$.
        Therefore, if $\rbodyB^l\eta'$ for $m'$ contains a null, so does $\rbodyB^l\eta$ for $m$.
        However, 
        we have already proven that the latter is not the case.
    \item[L\ref{alg:check_positive:modelsPsi2}:] 
        Here we may employ Lemma~\ref{Lemma_Completeness} again.
        We write $\Inter_b^m$ and $\Inter_b^{m'}$
        to distinguish the interpretations
        constructed in L\ref{alg:check_positive:Ib}
        of Algorithm~\ref{alg_check_positive}
        when called on $m$ and $m'$ respectively.
        We set $B_2^{m'} = \dom(m')$
        and $\bar{B}_2^{m'} = \rbody_2 \setminus B_2^{m'}$.
        While extending $m'$ to $m$,
        body atoms from $\bar{B}_2^{m'}$ get added to $B^m_2$.
        We define $B^\Delta_2 = \bar{B}_2^{m'} \cap B^m_2$.
        Then, $\bar{B}_2^{m'} = B^m_2 \cup B^\Delta_2$.
        In order to apply Lemma~\ref{Lemma_Completeness},
        we define $\mathcal{A} = \rbody_1 \cup \bar{B}_2^{m'} \cup \rhead_1\omega_\exists$,
        $\mathcal{B} = \rhead_2$ and $\mathcal{I} = \Inter_b^m$.
        From the definition of $\mathcal{A}$
        it is apparent that $\mathcal{A}\eta'\omega = \Inter_b^{m'}$.
        What needs to be shown is that $\eta\omega$
        is a homomorphism from $\mathcal{A}$ 
        to $\mathcal{I} = \Inter_b^m$.
        From the construction of $\Inter_b^m$
        we immediately obtain 
        that $(\rbody_1 \cup \rhead_1\omega_\exists \cup \bar{B}^m_2)\eta\omega \subseteq \Inter_b^m$.
        From $\eta$ being 
        a unifier between $B^m_2$ and $H^m_1$
        we get
        $B^\Delta_2 \eta\omega \subseteq B^m_2 \eta\omega
        \subseteq H^m_1 \omega_\exists \eta\omega \subseteq \Inter_b^m$.
        Therefore, $\mathcal{A} \eta\omega \subseteq \Inter_b^m$.
    \end{description}
    As every one of the above-mentioned checks fails on $m'$ and $\eta'$, 
    the algorithm either returns true and the computation finishes 
    or it goes on by extending $m'$ towards $m$ and ultimately accepting it.
    Hence, the algorithm is complete.
\end{proof}

\section{Proof of Theorem~\ref{theo_relres_correct}}

Beyond the proof of Theorem~\ref{theo_relres_correct}, we provide additional details on the case where a rule restrains itself, as outlined in the paper.

\thmRelResCorrect*
\begin{proof}
  Similar to the proof of Theorem~\ref{theo_relpos_correct}, 
  we separate our arguments into soundness and completeness.
  \paragraph{Soundness:} 
  The call to
  \extendres{$\arule_1,\arule_2,\emptyset$}${}$
  returns \textit{true} iff \checkres{$\arule_1,\arule_2,m,\eta$}${}=\textit{true}$ 
  for some atom mapping $m$ and mgu $\eta$, 
  meaning L\ref{alg:check_restraint:returnTrue} is reached in that call.
  We set $\rhead^m_2$, $\rhead^\ell_2$ and $\rhead^r_2$
  as in Algorithm~\ref{alg_check_restraint}
  when called on $m$ and $\eta$.
  In addition, we set $\rhead^m_1 = \im(m)$.

  L\ref{alg:check_restraint:assignIaTilde}
  constructs an interpretation $\tilde{\Inter}_a$
  from $\rbodyB\eta\omega$.
  This makes $\eta\omega$ a match for $\arule_2$ over $\tilde{\Inter}_a$
  that is unsatisfied due to the check in L\ref{alg:check_restraint:modelsPsi2}.
    L\ref{alg:check_restraint:assignIbTilde}
    builds the interpretation $\tilde{\mathcal{I}}_b = \mathcal{I}_a \cup (\rbodyA \cup \rhead^\ell_2\cup \rhead^r_2)\eta\omega$.
    By construction,
    $\eta\omega$ is a match for $\arule_1$ over that interpretation.
    It is also unsatisfied, 
    which results from the check in L\ref{alg:check_restraint:models}.
    We define the interpretation
    $\Inter_b$ as the result of applying 
    $\arule_1$ with the match $\eta\omega$,
    extending existential variables with their image under $\omega$.
    Note that $\tilde{\Inter}_a$ does not 
    contain nulls introduced by applying $\arule_1$ or $\arule_2$
    because of the checks 
    in L\ref{alg:check_restraint:noUniversalNull}
    and the fact that $\eta$ may not map anything to existential 
    variables (and hence no body variable can be mapped to a null by $\omega$).
    Similarly,
    $\tilde{\Inter}_b$ does not contain nulls 
    introduced from the application of $\arule_1$,
    which is implied by the checks 
    in L\ref{alg:check_restraint:noUniversalNull},
    L\ref{alg:check_restraint:noNullsPsi22}
    and L\ref{alg:check_restraint:noNullsPsi22_left}.
    In summary, we obtain interpretations $\Inter_a \subseteq \Inter_b$,
    such that $\omega_\exists\eta_\forall\omega_\forall$
    satisfies conditions (a) and (b) of Definition~\ref{def_restrained}.

    The alternative match is given by  
    $\eta^A\colon \rheadB\eta_\forall\omega \to \Inter_b$ 
    with $\eta^A(t) = t\eta\omega$.
    It is clear that 
    $\eta^A(t) = t$ for all terms in $\rbodyB\eta\omega$.
    The check in L\ref{alg:check_restraint:containsExists}
    ensures that $\eta^A$ maps at least one null 
    to some new term that is not present in $\rheadB\eta_\forall\omega$.
    By construction, 
    $(\rhead^\ell_2 \cup \rhead^r_2) \eta\omega$ is contained in $\Inter_b$.
    Furthermore,
    we have that 
    $(\rhead^m_2)\eta = (\rhead^m_1)\omega_\exists\eta$
    and therefore that 
    $(\rhead^m_2)\eta\omega = (\rhead^m_1)\omega_\exists\eta\omega_\forall \subseteq \Inter_b$.
    Thus we have $\rhead_2\eta\omega \subseteq \Inter_b$,
    which implies that $\eta^A$ is a homomorphism
    from $\rhead_2$ to $\Inter_b$.
    Note that $\eta^A$ is not an alternative match 
    over $\tilde{\Inter_b}$ because of the check in L\ref{alg:check_restraint:given}.

    \paragraph{Completeness:}
    To prove completeness,
    we assume that $\arule_1\relres\arule_2$,
    and hence that there are interpretations $\mathcal{J}_a \subset \mathcal{J}_b$
    and the functions $h_1$, $h_2$ and $h^A$
    satisfying the conditions of Definition \ref{def_restrained}.
    We may assume w.l.o.g. 
    that $h_1'$ and $h_2'$ 
    map every existential variable $v$
    in their domain to $\omega_\exists(v)$.
    Since $h^A$ 
    is an alternative match for $h'_2$ and $\arule_2$
    on $\mathcal{J}_b$
    but is not for $h_2$ and $\arule_2$ on 
    $\tilde{\mathcal{J}}_b = \mathcal{J}_b \setminus h_1'(\rhead_1)$,
    there must be 
    a partition $\rhead_2 = H^m_2 \mathrel{\dot\cup} \bar{H}^m_2$
    and a partition $\rhead_1 = H^m_1 \mathrel{\dot\cup} \bar{H}^m_1$
    such that 
    $h^A(h_2'(H^m_2)) = h_1'(H^m_1)$
    and $h^A(h_2'(\bar{H}^m_2)) \subseteq \tilde{\mathcal{J}}_b$.
    We define a substitution $h\colon \Clang \cup \Nlang \cup \Vlang \to \Clang \cup \Nlang \cup \Vlang$ as
    $$h(x) = \left\{ \begin{array}{ccl}
        h^A(h_2'(x)) && \text{if $x$ is a variable in } \arule_2 \\
        h_1'(x) && \text{if $x$ is a variable in } \arule_1 \\
        x && \text{otherwise}. \\
    \end{array}
    \right.$$
    The above function is well-defined because
    because $\arule_1$ and $\arule_2$ do not share any variables.
    By definition of $h$ we have
    $H^m_2 h = H^m_1 h$, implying
    that $H^m_2$ and $H^m_1$ are unifiable 
    and that there is an atom mapping $m$
    with $\dom(m) = H^m_2$ and $\im(m) = H^m_1$
    such that $h$ is a unifier of $m$.
    Hence we also obtain a most general unifier $\eta$
    of $m$.
    Since $\omega$ is assumed to assign terms only to
    constants and nulls not contained in $\arule_1$ or $\arule_2$,
    we can conclude that $\im(\eta) \cap \im(\omega) = \emptyset$
    and by Lemma~\ref{Lemma_Tau}
    that $h \subseteq \tau \circ (\eta\omega)$
    for some $\tau \colon \Clang \cup \Nlang \to \Clang \cup \Nlang$.

    In the following,
    we argue why each if-condition 
    in Algorithm \ref{alg_check_restraint} 
    when called on $m$ and $\eta$ fails.
    Every variable 
    assigned to a null by $\eta$
    must also be assigned to the same null 
    in $h$ since $\eta$
    is the most general unifier.
    But then either $\mathcal{J}_a$
    or $\tilde{\mathcal{J}}_b$
    would contain a null introduced by the application 
    of $\arule_1$.
    This follows from the fact that 
    $h$ is a homomorphism from $\rbody_2$ to $\mathcal{J}_a$ 
    and a homomorphism from $\rbody_1$ to $\tilde{\mathcal{J}}_b$.
    By this reasoning,
    the if-conditions on lines L\ref{alg:check_restraint:noUniversalNull},
    L\ref{alg:check_restraint:noNullsPsi22_left} and
    L\ref{alg:check_restraint:noNullsPsi22}
    all fail.

    To show that the check in L\ref{alg:check_restraint:containsExists} fails, 
    assume that $H^m_2$ does not contain any existential variables.
    Then $h^A(h_2'(H^m_2)) = h_2'(H^m_2) \subseteq \mathcal{J}_a \subset \tilde{\mathcal{J}}_b$
    as $h_2'(\rhead_2) \supseteq h_2'(H^m_2)$ results 
    from applying $\arule_2$ with match $h_2$.
    We further have $h^A(h_2'(\bar{H}^m_2)) \subseteq \tilde{\mathcal{J}}_b$.
    Overall this implies $h^A(h_2'(\rhead_2) \subseteq \tilde{\mathcal{J}}_b$,
    which contradicts condition (d) of Definition \ref{def_restrained}.

    We continue with the checks in L\ref{alg:check_restraint:modelsPsi2}
    and L\ref{alg:check_restraint:models}.
    We use Lemma \ref{Lemma_Completeness} in both cases 
    to show that if either one of the checks passes,
    then $h_1$ or $h_2$ would have been satisfied.
    For L\ref{alg:check_restraint:modelsPsi2}, 
    we define $\tilde{\mathcal{J}}_a = \mathcal{J}_a \setminus h_2'(\rhead_2)$.
    We have that $h$ 
    is an homomorphism from $\mathcal{A}_1 = \rbody_2$ to $\mathcal{I}_1 = \tilde{\mathcal{J}}_a$,
    since $h_2$ is a match for $\arule_2$.
    Also there is no extension of $h_2$ and therefore of $h$
    to a homomorphism from $\mathcal{B}_2 = \rhead_2$ to $\tilde{\mathcal{J}}_a$.
    Furthermore, we have $\tilde{\Inter}_a = \rbody_2 \eta\omega$.
    Therefore, we can use Lemma \ref{Lemma_Completeness}
    to show that $\tilde{\Inter}_a \not\models \exists \vec z.\ \rhead_2 \eta_\forall \omega_\forall$.
    A similar idea can be used 
    for the check in L\ref{alg:check_restraint:models}.
    This time, we set 
    $\mathcal{A}_2 = (\rbody_2 \cup \rbody_1 \cup \bar{H}^m_2 \cup \rhead_2 \omega_\exists)$,
    $\mathcal{B}_2 = \rhead_1$
    and $\mathcal{I}_2 = \tilde{\mathcal{J}}_b$.
    We have that $h(\rbody_2) = h_2(\rbody_2) \subseteq \tilde{\mathcal{J}}_a$
    because $h_2$ is a match for $\arule_2$;
    $h(\rbody_1) = h_1(\rbody_1) \subseteq \tilde{\mathcal{J}}_b$
    because $h_1$ is a match for $\arule_1$;
    $h(\bar{H}^m_2) = h^A(h_2'(\bar{H}^m_2)) \subseteq \tilde{\mathcal{J}}_b$
    by the initial assumption;
    and finally
    $h(\rhead_2 \omega_\exists) = h_1'(\rhead_2) \subseteq \tilde{\mathcal{J}}_b$
    because the result of applying $\arule_2$ is contained in $\tilde{\mathcal{J}}_b$.
    It is now easy to see
    that $\mathcal{A}_2 \eta\omega = \tilde{\Inter}_b$.

    For the check in line L\ref{alg:check_restraint:given}
    observe that $\rhead_2\eta\omega \subseteq \mathcal{A}_2\eta\omega$.
    It follows from that $\rhead_2(\tau \circ (\eta\omega)) = \rhead_2 h = h^A(h_2(\rhead_2)) 
    \subseteq \mathcal{A}_2(\tau \circ (\eta\omega)) \subseteq \tilde{\mathcal{J}}_b$.
    But this would contradict condition (d) of Definition \ref{def_restrained}.

    As in Theorem~\ref{theo_relres_correct}, it remains to be shown that the iteration in function 
    \extendres{$\arule_1,\arule_2,\emptyset$} 
    reaches the postulated mapping $m$
    or returns \textit{true} earlier.
    Let $m'$ be any atom mapping that can be extended to $m$.
    Because $\eta$ is a unifier for $m$,
    $\eta$ is also a unifier for $m'$.
    This implies that there is a most general unifier $\eta'$ for $m'$
    and by Lemma \ref{Lemma_Tau}.
    Therefore $\eta \subseteq \tau' \circ (\eta'\omega)$
    for some $\tau' \colon \Clang \cup \Nlang \to \Clang \cup \Nlang$.
    We now need to argue that Algorithm~\ref{alg_check_restraint} 
    does not return \textit{false} on $m'$.
    \begin{description}
        \item[L\ref{alg:check_restraint:noUniversalNull}:] 
            Any universal variable assigned to a null by $\eta'$ 
            must also be assigned to the same null by $\eta$,
            since $\eta'$ is more general.
            Because we already know that this check fails for $m$,
            we can conclude that it fails for $m'$ as well.
        \item[L\ref{alg:check_restraint:noNullsPsi22_left}:]
            From the way an atom mapping is extended
            by the modification of Algorithm~\ref{alg_extend_positive},
            we know that if an atom is contained in $\rhead^\ell_2$
            for $m'$, then it is also contained in $\rhead^\ell_2$
            for $m$. Hence, this if-check would also have to fail for $m$,
            which we already ruled out.
        \item[L\ref{alg:check_restraint:modelsPsi2}:]
            We set $\mathcal{A} = \rbody_2$,
            $\mathcal{B} = \rhead_2$ and $\mathcal{I} = \tilde{\Inter}_a$.
            Then, $\eta$ is a homomorphism from $\mathcal{A}$
            to $\mathcal{I}$
            that cannot be extended to a homomorphism from $\mathcal{B}$ to $\mathcal{I}$.
            From Lemma~\ref{Lemma_Completeness}
            we immediately obtain 
            $\rbody_2 \eta'\omega \not\models \exists \vec z.\ \rhead_2 \eta'_\forall\omega_\forall$.
            Therefore this check fails.\qed
    \end{description}

\end{proof}

\SetKwFunction{checkresself}{check$^\square_{\textsf{self}}$}
\SetKwFunction{extendresself}{extend$^\square_{\textsf{self}}$}
\begin{algorithm}[t]\caption{\FuncSty{check$^\square_{\textsf{self}}$}}\label{alg_check_self}

\KwIn{rules $\arule:\rbody\to\exists\vec{v}.\rhead$, atom mapping $m$ with mgu $\eta$}
\KwOut{\textit{true} if a self-restraint is found for $m$}

$\rhead^m \leftarrow \dom(m)$\;
$\rhead^\ell \leftarrow \ghost{$\{\rhead[j]\in(\rhead{\setminus}\,\rhead^m)\mid j\,{<}\,\maxidx{m}\}$}$\;
$\rhead^r \leftarrow \ghost{$\{\rhead[j]\in(\rhead{\setminus}\,\rhead^m)\mid j\,{>}\,\maxidx{m}\}$}$\;

\lIf{$\rhead\eta_\exists = \rhead\omega_\exists$}{\label{alg:check_self:reducing}%
    \KwRet{false}
}

\lIf{$x\eta\in\Nlang$ for some $x\in\Vlang_\forall$}{\label{alg:check_self:noUniversalNull}%
    \KwRet{false}
}

\If{$z\eta\in\Nlang$ for some $z\in\Vlang_\exists$ in $\rhead^\ell$}{\label{alg:check_self:noExistentialNull_left}%
    \KwRet{false}\;
}

\If{$z\eta\in\Nlang$ for some $z\in\Vlang_\exists$ in $\rhead^r$}{\label{alg:check_self:noExistentialNull_right}%
    \KwRet{\extendresself{$\arule$, $m$}}\;
}

$\tilde{\mathcal{I}} \leftarrow (\rbody \cup \rhead^\ell \cup \rhead^r)\eta\omega$\;\label{alg:check_self:ITilde}

\lIf{$\tilde{\mathcal{I}} \models \exists\vec{v}.\, \rhead\eta_\forall\omega_\forall$}{\label{alg:check_self:models}%
    \KwRet{\extendresself{$\arule$, $m$}}
}

\KwRet{true}\;
\end{algorithm}
Algorithm~\ref{alg_check_self} specifies the central function that we use for checking the special case
where a rule restrains itself through a single rule application (rather than two distinct applications
as considered before).
\extendresself{$\arule$,$m$} works the 
same way as the regular \extendres{$\arule_1$, $\arule_2$, $m$} function.
However, 
since we are dealing with only a single rule application 
now, no renaming of variables is required.
However, head atoms in the domain of the atom mapping may still contain existential variables,
whereas those in its range have such variables replaced by nulls.
The essential correctness result for the self-restraining case is as follows:

\begin{theorem}\label{theo_relres_self_correct}
Given a rule $\arule$,
\extendresself{$\arule$,$\emptyset$}${}=\textit{true}$ iff $\arule\relres\arule$
holds according to Definition~\ref{def_restrained} for some $\Inter_a=\Inter_b$. 
\end{theorem}
\begin{proof}
  As before, 
  we divide our argument for soundness and completeness.
  \paragraph{Soundness:}
  Assume that \checkresself{$\arule, m, \eta$}{}$=\textit{true}$ 
  for some atom mapping $m$ and mgu $\eta$.
  We define $\omega'$ to be a substitution
  mapping variables to the same terms as $\omega$
  except for nulls that do not appear in $\im(\eta)$,
  which are assigned to unique constants instead.
  We set $\tilde{\Inter} = (\rbody \cup \rhead^\ell \cup \rhead^r)\eta\omega'$
  similarly as in L\ref{alg:check_self:ITilde}.
  Furthermore,
  let $\Inter = \tilde{\Inter} \cup \rhead \omega_\exists \eta_\forall\omega_\forall$
  be an interpretation.
  By construction of $\tilde{\Inter}$
  we have that $\eta_\forall\omega_\forall$
  is a match for $\tilde{\Inter}$. 
  The check in L\ref{alg:check_self:models}
  ensures that it is unsatisfied.
  Note that $\tilde{\Inter}$ does not contain any nulls 
  introduced by applying $\arule$
  because of the checks in L\ref{alg:check_self:noUniversalNull},
  L\ref{alg:check_self:noExistentialNull_left}, L\ref{alg:check_self:noExistentialNull_right}
  and our definition of $\omega'$.
  The alternative match is given by  
  $\eta^A\colon \rhead\eta_\forall\omega' \to \Inter_b$ 
  with $\eta^A(t) = t\eta\omega'$.
  From the check in L\ref{alg:check_self:reducing}
  it follows $\eta^A$ maps a null not present in $\eta_\forall\omega$.
  Therefore we have $\mathcal{I}_a = \mathcal{I}_b = \mathcal{I}$
  and the functions $\eta_\forall\omega_\forall$ and 
  the alternative match $\eta^A$
  satisfying conditions (a), (b) and (c) of Definition~\ref{def_restrained}.
  Condition (d) does not need to be verified
  since there cannot be an alternative match for $\arule$
  before its application.

  \paragraph{Completeness:}
  Completeness can be handled with similar arguments
  as in Theorem~\ref{theo_relres_correct}.
  Here, we briefly describe how to apply Lemma~\ref{Lemma_Completeness}
  for L\ref{alg:check_self:models}.
  We set $\mathcal{A} = \rbody \cup \rhead^\ell \cup \rhead^r$,
  $\mathcal{B} = \rhead$ and $\mathcal{I} = \tilde{\Inter}$.\qed
\end{proof}

\end{tr}

\end{document}